\documentclass[opre,nonblindrev]{informs4}
\OneAndAHalfSpacedXI % current default line spacing
\usepackage{natbib}
 \bibpunct[, ]{(}{)}{,}{a}{}{,}%
 \def\bibfont{\small}%
\usepackage[table]{xcolor}
\usepackage{amsmath}
\usepackage{amssymb}
\usepackage[linesnumbered,ruled,vlined]{algorithm2e}
\usepackage{wrapfig}

\usepackage{booktabs}
\usepackage{multirow, enumitem}
\usepackage{bm,amsmath,amssymb,amsfonts,mathrsfs,graphicx,caption,bbold,thm-restate,multirow,xspace,ifthen}
\usepackage[ruled,vlined]{algorithm2e}
\usepackage[noend]{algpseudocode}

\usepackage{float} % for H specifier

\usepackage{color-edits}
\addauthor[David]{dk}{red}
\addauthor[Negin]{ng}{magenta}
\addauthor[Susan]{fs}{green}
\addauthor[Ehsan]{ee}{cyan}
\addauthor[Revision]{r}{blue}

\providecommand{\Kth}[1]{\ensuremath{#1^{\text{th}}}\xspace}
\providecommand{\SET}[1]{\ensuremath{\left\{ #1 \right\}}\xspace}
\providecommand{\Set}[2]{\ensuremath{\SET{#1 \mid #2}}\xspace}

\providecommand{\PROB}{\ensuremath{{\mathbb{P}}}\xspace}
\providecommand{\Prob}[2][]{\ensuremath{%
\ifthenelse{\equal{#1}{}}{\PROB[#2]}{\PROB_{#1}[#2]}}\xspace}

\providecommand{\Expect}[2][]{\ensuremath{%
\ifthenelse{\equal{#1}{}}{\mathbb{E}}{\mathbb{E}_{#1}}%
\left[#2\right]}\xspace}

\providecommand{\Event}[1]{\ensuremath{\mathcal{E}_{#1}}\xspace}

\def\QED{{\phantom{x}} \hfill \ensuremath{\rule{1.3ex}{1.3ex}}}

\newcommand{\universe}{\mathcal{N}}
\newcommand{\ybf}{\mathbf{y}}
\newcommand{\cons}{\nu}

\newcommand{\raresum}{\rho}
\newcommand{\greedymincover}{\textsc{GreedyMinCover}}
\newcommand{\freqpi}{\Pi^F}
\newcommand{\rarepi}{\Pi^R}
\newcommand{\choicetype}[2]{\Theta_{#1}(#2)}
\newcommand{\edgetype}[2]{\Theta_{#1 \oplus #2}}

\newcommand{\AlgDAG}{\textsc{AlgDAG}\xspace}
\newcommand{\AlgIE}{\textsc{AlgIE}\xspace}
\newcommand{\hatG}{\widehat{G}}
\newcommand{\hatE}{\widehat{E}}
\newcommand{\hatV}{\widehat{V}}
\newcommand{\hatp}{\widehat{p}}
\newcommand{\hatpv}{\widehat{\bm{p}}}
\newcommand{\hatq}{\widehat{q}}
\newcommand{\hatev}{\widehat{\bm{e}}}
\newcommand{\numel}{n}
\newcommand{\numitems}{n_0}
\newcommand{\numsamples}{m}
\newcommand{\probboundc}{c}

\newcommand{\familysetnodepre}{\mathcal{{P}}}
\newcommand{\familysetnodeprealg}{\widehat{\mathcal{P}}}
\newcommand{\familysetnodepreF}{\mathcal{{P}}^{\textrm{F}}}

\newcommand{\coveralg}{\widehat{\mathcal{C}}}
\newcommand{\offerset}{S}
\newcommand{\familyminus}{\mathcal{S}^{-}}
\newcommand{\collectsets}{\mathcal{L}}
\newcommand{\edgeprob}[3][]{%
    e_{#2 \oplus #3}^{#1}
}
\newcommand{\hatedgeprob}[3][]{%
    \widehat{e}_{#2 \oplus #3}^{#1}
}
\usepackage{lipsum} 
\newcommand{\toppi}{A}

\newcommand{\GENMODEL}{General Model\xspace}
\newcommand{\genmodel}[2]{$(#1,#2)$-\GENMODEL}
\newcommand{\RANDMODEL}{Random Model\xspace}
\newcommand{\randmodel}[2]{$(#1,#2)$-\RANDMODEL}

\TheoremsNumberedThrough     % Preferred (Theorem 1, Lemma 1, Theorem 2)
\ECRepeatTheorems

\EquationsNumberedThrough    % Default: (1), (2), ...
\begin{document}
%%%%%%%%%%%%%%%%

%\setcounter{page}{1}
%\DOI{10.1287/xxxx.0000.0000}%


\def\otherDAG{\textbf{Other uses of DAGs in non-parametric choice models.} In this work, one of our key contributions is the use of DAGs to represent non-parametric choice models. Our use of DAGs differs significantly from that in past work on choice models \citep{jagabathula2018partial, jagabathula2022personalized}. In the model presented by \citet{jagabathula2018partial}, each customer type is characterized by a partial order over items, representing a collection of pairwise preference relations. This partial order is conveniently visualized using a DAG, with each node corresponding to an item. A directed edge from item $a$ to item $b$ indicates that item $a$ is preferred over item $b$ in the partial order. The full ranking of a customer is drawn randomly from a distribution specific to the customer's segment, consistent with the partial order. The customer then selects the highest-ranked item from their consideration set. The primary goal of the research in \citet{jagabathula2018partial} is to estimate this model using offline data, in contrast to the active learning approach in our work. \citet{jagabathula2022personalized} extend the DAG concept in \cite{jagabathula2018partial} to include promotions. Here, each item is represented by two nodes: one for its full price and another for its discounted version, thus optimizing promotion planning through this enhanced DAG framework.

Our use of DAGs as a modeling tool is fundamentally different from their uses, in that in our setting, the DAG is used to represent the choice model of all customer types at once, rather than just one customer type. That is, in our definition of DAGs, by labeling nodes with sets (rather than items), we manage to represent the entire choice model, rather than partial orders. }

\def\willMa{Another  
class of models, which also overcome the IIA restriction, is based on Markov Chains {\citep{blanchet2016markov, feldman2017revenue, csimcsek2018expectation, ma2023assortment}}.
Under such models, the distribution over rankings $\pi$ is defined as follows. The states of the Markov Chain are the items (plus a no-purchase option). The top choice is the starting node, drawn from a distribution $(\lambda_j)$ over states $j$. Subsequently, the next state is always a node $i$ chosen with conditional probability $\rho_{j,i}$. The ranking $\pi$ is obtained as the order in which nodes were visited for the first time (since they will typically be visited multiple times).
This model could in principle be encoded in a DAG --- however, because all types have positive probability of occurring, the corresponding DAG would have $2^n$ nodes.} 
\RUNAUTHOR{}

\RUNTITLE{Active Learning for Non-Parametric Choice Models}

% \TITLE{Bundling Information Goods of Decreasing Value}
\TITLE{Active Learning for Non-Parametric Choice Models}

\ARTICLEAUTHORS{%
\AUTHOR{Fransisca Susan, Negin Golrezaei}
\AFF{MIT Sloan School of Management, Operations Management, \EMAIL{fsusan@mit.edu, golrezae@mit.edu}}
\AUTHOR{Ehsan Emamjomeh-Zadeh}
\AFF{Meta Platforms, Inc., \EMAIL{ehsan7069@gmail.com}}
\AUTHOR{David Kempe}
\AFF{University of Southern California, Los Angeles, \EMAIL{david.m.kempe@gmail.com}}
} % end of the block

\ABSTRACT{

{We study the problem of actively learning a non-parametric choice model based on consumers' decisions. We present a negative result showing that such choice models may not be identifiable. To overcome the identifiability problem, we introduce a directed acyclic graph (DAG) representation of the choice model. This representation provably encodes all the information about the choice model which can be inferred from the available data, in the sense that it permits computing all choice probabilities. 

 We establish that given exact choice probabilities for a collection of item sets, one can reconstruct the DAG.  However, attempting to extend this methodology to estimate the DAG from noisy choice frequency data obtained during an active learning process leads to inaccuracies. To address this challenge, we present an inclusion-exclusion approach that effectively manages error propagation across DAG levels, leading to a more accurate estimate of the DAG.
Utilizing this technique, our algorithm estimates the DAG representation of an underlying non-parametric choice model. The algorithm operates efficiently (in polynomial time)  when the set of frequent rankings is drawn uniformly at random.  It learns the distribution over the most popular items among frequent preference types by actively and repeatedly offering assortments of items and observing the chosen item. We demonstrate that our algorithm more effectively recovers a set of frequent preferences on both synthetic and publicly available datasets on consumers' preferences, compared to corresponding non-active learning estimation algorithms. These findings underscore the value of our algorithm and the broader applicability of active-learning approaches in modeling consumer behavior.}

%We study the problem of actively learning a non-parametric choice model based on consumers' decisions. We present a negative result showing that such choice models may not be identifiable. To overcome the identifiability problem, we introduce a directed acyclic graph (DAG) representation of the choice model, which in a sense captures as much information about the choice model as could information-theoretically be identified. We then consider the problem of learning an approximation to this DAG representation in an active-learning setting. We design an efficient active-learning algorithm to estimate the DAG representation of the non-parametric choice model, which runs in polynomial time when the set of frequent rankings is drawn uniformly at random. Our algorithm learns the distribution over the most popular items of frequent preferences by actively and repeatedly offering assortments of items and observing the item chosen. We show that our algorithm can better recover a set of frequent preferences on both a synthetic and publicly available dataset on consumers' preferences, compared to the corresponding non-active learning estimation algorithms. This demonstrates the value of our algorithm and active-learning approaches more generally.
}

\KEYWORDS{Non-parametric choice models, Active learning, DAG representation}

% \KEYWORDS{Non-parametric choice model, Product ranking, Active learning} 
% \HISTORY{This paper was first submitted on April 12, 1922 and has been with the authors for 83 years for 65 revisions.}

\maketitle
%%%%%%%%%%%%%%%%%%%%%%%%%%%%%%%%%%%%%%%%%%%%%%%%%%%%%%%%%%%%%%%%%%%%%%

\section{Introduction}
\label{sec:intro}
Choice models are used by firms to capture consumers' preferences, and predict their decisions. They can help with important operational decisions, including demand forecasting (e.g., \cite{mcfadden1977demand}, \cite{mcgill1999revenue}), inventory planning (e.g., \cite{ryzin1999relationship}, \cite{gaur2006assortment, aouad2019approximation}), assortment optimization (e.g., \cite{talluri2004revenue}, \cite{rusmevichientong2010dynamic}, \cite{davis2014assortment, golrezaei2014real}), and product ranking optimization (e.g., \cite{derakhshan2020product, niazadeh2021online, golrezaei2021learning}). Among choice models, parametric choice models --- and in particular random utility maximization models --- have received the most attention by far. While parametric choice models
provide concise representations of consumer preferences and decisions, they impose certain structures on consumer preferences, which may not be valid (\cite{jagabathula2019limit}). If the structures are invalid, imposing them can  lead to model misspecification and inaccurate decision making. These drawbacks have motivated the  introduction and study 
 of non-parametric choice models \citep{farias2013nonparametric, jagabathula2017nonparametric, paul2018assortment}). 
Non-parametric choice models, however, are challenging to estimate using only offline collected transaction data since such data may not have enough variation in the offered sets of items.  

In many settings, particularly on online platforms, one can go beyond using offline data. 
On these platforms, one can influence the data collection process, obtaining data sets more suitable for estimating non-parametric choice models. The process of influencing data collection, which is known as \emph{active learning}, has been widely used in various contexts; see \cite{settles2009active} and \cite{7bb36d742bf34b9d854aa0778ec47903} for surveys, and Section~\ref{sec:related-work} for a more detailed discussion of related work. In this process, the platform can dynamically change the set of items offered to arriving consumers and observe their choice. However, it is still unclear whether active learning can fully overcome the challenge of estimating non-parametric choice models. This motivates our main research question: \emph{How can we estimate a non-parametric choice model by relying on active learning? Can we estimate non-parametric choice models with a polynomial-size dataset, which is obtained through an efficient active learning process?}   

To answer these questions, we consider an online platform with $\numel$ items. The platform faces distinct types of consumers: each consumer type is characterized by a  preference ranking over the $\numel$ items and the fraction of the population having this type. 
Consumers choose according to their types/rankings, and their types are not observable by the platform. The model is presented more formally and in detail in Section~\ref{sec:model}.

To learn the probability distribution over the rankings, the platform engages in an active learning process. In every round, a consumer (whose unobservable type is drawn based on the unknown population distribution) arrives. Upon the arrival of the consumer, based on past observations, the platform decides on the set (assortment) of items to offer. The consumer then chooses the item she ranks highest among the items offered. The platform observes the choice made by the consumer; then, the process repeats with another consumer. 

\subsection{Our Contributions}
\textbf{Indistinguishability.} One of the main challenges in learning non-parametric choice models is the identifiability problem: the data may not uniquely identify the choice model. This problem clearly exists when the choice model is estimated from \emph{offline} transaction data that contains the choices made by consumers only for some pre-specified offered sets; for example, if all consumers were offered the same set $S$, the algorithm can only learn the distribution of the most preferred item in $S$. 
 With active learning, however, the algorithm can introduce heterogeneity in the offered sets. This raises the following question: Is it possible to learn \emph{every} non-parametric choice model using active learning? In Section~\ref{sec:indistinguishable}, we show that even with active learning, regardless of how many consumers the algorithm interacts with, the choice model may not be uniquely identifiable. In particular, we determine simple conditions under which two different non-parametric choice models are information-theoretically \emph{indistinguishable} from each other, no matter how many and what assortments are offered to arriving consumers.

\textbf{Directed Acyclic Graph representations of non-parametric choice models.} 
In light of our indistinguishability result, in Section~\ref{sec:dag-rep}, we provide a novel representation of non-parametric choice models. This representation, which we refer to as Directed Acyclic Graph (DAG) representation, can always be uniquely identified, assuming enough samples with suitably chosen choice sets. The DAG contains a node for each set $A$ of items such that at least one consumer type ranks all items in $A$ (in any order) ahead of all items not in $A$; thus, the nodes correspond to all possible prefixes of consumer rankings. The node $A$ is labeled with the total probability of all types whose rankings have $A$ as a prefix. All edges are of the form $(A, A \cup \SET{z})$ for an item $z \notin A$, and exist whenever there exists a consumer type ranking all of $A$ in the first $|A|$ positions (in any order), followed by $z$ as the next item.
For any $j$, the distribution over sets $A$ of cardinality $|A| = j$ captures the distribution over top-$j$ most preferred item sets; hence, it can be used in various decision-making processes, including inventory planning and product design decisions.
{More detail on such uses can be found in Section~\ref{sec:approximation-implications}.}

{\textbf{Computing the choice probabilities using the DAG.}
One of the key reasons for using the DAG representation is that in addition to giving an intuitive visual representation with good downstream processing properties, it still fully encodes all available information about the choice model. Specifically, as we will demonstrate in  Section~\ref{sec:choice_DAG}, when provided with a DAG representing a non-parametric choice model, one can --- in polynomial time --- calculate the probability of selecting an item $z$ from a given set $S$ of items. Since the only source of information for constructing the DAG representation is these choice frequencies, the DAG representation encapsulates as much information as can be gleaned from the available data.}

{\textbf{Computing the DAG using  exact choice probabilities.} Given the desirability of having access to a DAG representation, we show --- in Section~\ref{sec:construct_DAG_exact} --- that given the exact choice probabilities, an efficient algorithm can construct the DAG. The algorithm constructs the DAG iteratively, by increasing size of the prefix $A$. When computing the relevant probabilities for larger prefixes $A$, it uses the probabilities of smaller prefixes $A'$ whose probabilities were computed in earlier iterations.}

{\textbf{Active learning of the DAG representation.} 
Section~\ref{sec:construct_DAG_exact} shows how to construct the DAG when exact choice probabilities are available. In practice, however, exact probabilities are inaccessible --- instead, choice \emph{frequencies} will be inferred from data. As we show, even small misestimates of choice frequencies have the potential to accumulate and lead to significant misestimates of the model. 
Concretely, this issue arises because probability estimates for larger prefixes use the estimates for smaller prefixes, which in turn use estimates for even smaller prefixes, etc., resulting in an inclusion-exclusion calculation. This type of dependency can lead to exponential amplification of estimation errors.

Our primary technical contribution lies in our  method of selecting a significantly smaller and less ``obvious'' collection of sets for inclusion-exclusion operations. We achieve this by employing a set cover approximation algorithm on carefully chosen prefixes of $A$. When the set covers remain sufficiently small (of logarithmic size), the algorithm can identify the correct DAG with high accuracy using only a polynomial number of active queries, while still leading to only limited error propagation across different levels. Among others, the set covers are logarithmically small in the number $\numel$ of items with high probability when the frequent items are uniformly random.
When the set covers become larger, our algorithm exhibits adaptability. It can respond by using additional queries or by diagnosing situations where outputs obtained with few queries may not be reliable. This distinguishing feature sets our approach apart from parametric models, under which inaccuracies may arise if modeling assumptions are violated. The algorithm and its detailed analysis  are comprehensively discussed in Section~\ref{sec:dag-algo}.}

\textbf{Value of active learning via a case study.} To evaluate the accuracy and simplicity of our algorithm, we test it empirically on synthetically generated  choice models as well as ranked preference parameters inferred from a Sushi preference data set (\cite{kamishima2003nantonac}, \cite{kamishima2005supervised}); the experiments are presented in Section~\ref{sec:experiment}. We show that empirically, the number of data points needed to estimate the DAG representation of a choice model accurately is smaller than the theoretical bound, which shows the robustness of our algorithm. We also observe that our algorithm significantly outperforms a non-active choice model estimation algorithm (\cite{farias2013nonparametric}), while also using a smaller number of queries. This shows the value of our algorithm and the active learning approach in estimation.

\section{Related Work}
\label{sec:related-work}
\textbf{Parametric choice models.} There are two general approaches to modeling consumer preferences: parametric and non-parametric models. Parametric choice models specify some functional form that connects related attributes and price to utility values and choice probabilities. One example is a choice model that assumes independent demand for each product; such a model does not capture substitution effects between similar products in the offered assortment \citep{mahajan1999retail}. Another example is the Multinomial Logit model (MNL), a parametric choice model that is commonly used in marketing, economics (see \cite{ben1985discrete} and \cite{mahajan1999retail}), and revenue management \citep{mcfadden1973conditional,ben1985discrete}. MNL has the IIA (independence of irrelevant alternatives) property: the odds of choosing one item over another do not depend on whether or not a third item is present in the assortment. This property is frequently unrealistic, especially when products exhibit complementarities. Other examples of parametric choice models, such as the generalized nested logit model \citep{wen2001generalized} and mixed logit model \citep{ben1985discrete}, overcome the IIA restriction.

{Another  
class of models, which also overcome the IIA restriction, is based on Markov Chains\footnote{Somewhat related are cascade choice models \citep{aggarwal:feldman:muthukrishnan:pal,kempe2008cascade, golrezaei2021learning}, originally defined to model click probabilities in sponsored search ads. Cascade models capture a consumer going through items sequentially until choosing an item to buy or exiting the process.} {\citep{blanchet2016markov, feldman2017revenue, csimcsek2018expectation, ma2023assortment}}.
Under such models, the distribution over rankings $\pi$ is defined as follows. The states of the Markov Chain are the items (plus a no-purchase option). The top choice is the starting node, drawn from a distribution $(\lambda_j)$ over states $j$. Subsequently, the next state is always a node $i$ chosen with conditional probability $\rho_{j,i}$. The ranking $\pi$ is obtained as the order in which nodes were visited for the first time (since they will typically be visited multiple times).
This model could in principle be encoded in a DAG --- however, because all types have positive probability of occurring, the corresponding DAG would have $2^n$ nodes. }

Historically, there is a vast literature on optimizing the assortment and price based on some parametric choice models. This includes the seminal work by \cite{talluri2004revenue} that infers a revenue-ordered set (a set containing a certain number of items with the highest revenue) as an optimal assortment for the deterministic MNL model, the work of \cite{rusmevichientong2010dynamic} determining the optimal assortment for the MNL model with a capacity constraint, the work of \cite{rusmevichientong2012robust} who develop an optimal assortment for a robust set of likely parameters of choice models, and the work of \cite{davis2014assortment} who optimize the assortment for the nested logit model. 
Implicit in this general approach is the caveat that one should fit the \emph{right} parametric choice model to data before making predictions or decisions; this is a difficult problem since the implicit pre-specified structures in the parametric choice models might not be true in practice. There is a huge risk of model mis-specification, which will lead to inaccurate decision making downstream. Thus, there is a trade-off between specification and estimation:  while a complex model can approximate a wider range of choice behaviors and hence may have smaller specification error, in return, it may have big estimation errors when estimated {without sufficient} data.

\textbf{Non-parametric choice models.} Non-parametric choice models consider distributions over all rankings. They have risen in popularity due to the increased availability of data; see \cite{farias2013nonparametric}, \cite{van2015market}, \cite{van2017expectation}, \cite{haensel2011estimating}, \cite{jagabathula2017nonparametric}, \cite{aouad2019approximation}, \cite{honhon2012optimal}, and \cite{rusmevichientong2010dynamic}. 
Our work is most similar to that of \cite{farias2013nonparametric}, \cite{van2015market} and \cite{van2017expectation}, although all of them only use \emph{offline} transaction data to estimate the best fitting choice model. Specifically, \cite{farias2013nonparametric} estimate the sparsest choice model consistent with the available transaction data; \cite{van2015market} start from a parsimonious set of preferences, then iteratively add new preferences that increase the likelihood value of the data using a column generation based procedure; \cite{van2015market} develop an efficient and easy-to-implement expectation maximization method that finds the non-parametric choice model with the highest likelihood. 

Our work is different in one main aspect: we assume that the algorithm gets to determine the assortments offered to consumers. This allows the algorithm to acquire the most useful data, by querying the right assortments. As a result, our algorithm overcomes the identifiability issues that arise when one estimates the choice model using offline data. We highlight that instead, \cite{van2015market} and \cite{farias2013nonparametric} aim to overcome this issue by imposing \emph{assumptions} on the observed data. In \cite{van2015market} and \cite{van2017expectation}, the authors  impose that for each offer set $\offerset$ and product $z$ in the set, there exists a preference ranking under which $z$ ranks highest in $S$; thus, they only consider items that are preferred the most by at least one type. In contrast, we consider all items and learn the set of the most popular items. \cite{farias2013nonparametric} impose the assumption that for every preference appearing in the population, there exists a query and an item in the data such that the item ranks the highest in the query only for that preference, and that the set of probabilities is linearly independent with respect to the integers. 
We do not need to make any such assumptions, {because} the DAG representation is unique given the transaction data.

There is a long line of work that aims to design algorithms identifying optimal or near-optimal revenue maximizing assortments, under the assumption that the model primitives are known. These studies are usually accompanied by a case study in which the model primitives are estimated from either real or synthetic data sets. The case studies are usually used to evaluate the proposed assortment planning algorithms; see, for example, \citet{haensel2011estimating, honhon2012optimal, jagabathula2017nonparametric, aouad2018approximability, aouad2019approximation, feldman2019assortment, derakhshan2020product} for some of the work that follows this approach for non-parametric choice models. To estimate the choice models, these approaches mostly use maximum likelihood estimators (MLE): the locally optimal solutions are obtained via expectation maximization (EM) algorithms. Notice that the estimation process in these works tends to be conducted under  restrictive assumptions. While these assumptions are mainly imposed to ensure the tractability of the assortment planning problem, the assumptions can simplify the estimation process as well. For example, \cite{feldman2019assortment} assume that consumer types are derived from paths in a tree, where the set of preference lists consists of ordered nodes visited along each path in the tree. This assumption allows them to use MLE on a domain of polynomial support size, to derive a fitted general tree representing the choice model while only considering linear paths on the tree. \cite{honhon2012optimal} assume that all the products can be mapped onto a hierarchical ordering system, such as branched, vertical or horizontal order; they then use a dynamic programming based algorithm to find the optimal assortment.

{\textbf{Other uses of DAGs in non-parametric choice models.} In this work, one of our key contributions is the use of DAGs to represent non-parametric choice models. Our use of DAGs differs significantly from that in past work on choice models \citep{jagabathula2018partial, jagabathula2022personalized}. In the model presented by \citet{jagabathula2018partial}, each customer type is characterized by a partial order over items, representing a collection of pairwise preference relations. This partial order is conveniently visualized using a DAG, with each node corresponding to an item. A directed edge from item $a$ to item $b$ indicates that item $a$ is preferred over item $b$ in the partial order. The full ranking of a customer is drawn randomly from a distribution specific to the customer's segment, consistent with the partial order. The customer then selects the highest-ranked item from their consideration set. The primary goal of the research in \citet{jagabathula2018partial} is to estimate this model using offline data, in contrast to the active learning approach in our work. \citet{jagabathula2022personalized} extend the DAG concept in \cite{jagabathula2018partial} to include promotions. Here, each item is represented by two nodes: one for its full price and another for its discounted version, thus optimizing promotion planning through this enhanced DAG framework.

Our use of DAGs as a modeling tool is fundamentally different from their uses, in that in our setting, the DAG is used to represent the choice model of all customer types at once, rather than just one customer type. That is, in our definition of DAGs, by labeling nodes with sets (rather than items), we manage to represent the entire choice model, rather than partial orders. 
}

\textbf{Online learning for assortment planning.} The literature on online learning for assortment planning is also related to our work. In this literature, the goal is to learn the consumers' preferences from their actions, in order to identify the revenue-optimal assortments. This is mainly done for parametric choice models such as the multinomial logit, nested logit,  Markov Chain, and cascade  choice models as seen, for example, in \citet{rusmevichientong2010dynamic,saure2013optimal, agrawal2019mnl,chen2021dynamic, niazadeh2020online,gallego2021optimal, golrezaei2022learning}.
While these works mostly care about identifying the optimal assortment and minimizing regret, we focus on estimating the choice model itself. Furthermore, online learning has, to the best of our knowledge, not yet been studied for non-parametric choice models.

\textbf{Active learning.} Active learning, sometimes also known as ``query learning'' or ``optimal experimental design'' in statistics, is a field of machine learning in which the learning algorithm gets to \emph{choose} the queries (unlabeled data), to be labeled by an oracle which knows the true labels. This control over the training data typically improves the performance of the algorithm with less training data. Within machine learning, the idea of active or online learning was introduced in the seminal papers of \cite{angluin:1988:queries-concept} and
\cite{littlestone:1988:online-learning}, in the context of learning a binary classifier. Naturally, combining offline or unlabeled data with data obtained online is frequently the best or most natural choice in practice; see, e.g., \cite{cohn1996active}. Although active learning has been well studied in the literature (see \citet{settles2009active} and \citet{7bb36d742bf34b9d854aa0778ec47903} for surveys), to the best of our knowledge, active learning has not been used for learning choice models. Previously, several works in the literature have explored and shown the value of active learning compared to its ``passive'' counterpart. For example, \citet{zheng2006selectively} developed a new active learning technique for an information acquisition problem in order to, for example, predict the probability of purchase and credit default rate. They demonstrated that the proposed method performs well empirically. Similarly, \citet{aviv2002pricing} explored the value of proactively setting prices to impact the revenue.

A line of work similar to active learning is the dynamic sampling literature where one dynamically selects samples to learn parameters while optimizing the objective functions; see \cite{shi2021dynamic}, \cite{zhang2020sequential}, \cite{shin2018tractable}.  \cite{shi2021dynamic} dynamically allocate samples in a finite sampling budget to learn the system's feasible alternatives efficiently in a feasibility determination problem appearing in many applications, such as call center design and hospital resource allocation. \cite{shin2018tractable} intelligently allocate samples in their simulation to minimize the probability of selecting a system that does not have the highest mean out of several competing alternatives (``systems''), when the probability distribution determining each system's performance is unknown but can be learned from a limited number of samples they can obtain. However, none of these dynamic sampling approaches was done in the assortment planning or choice modeling setting.

\section{Model}
\label{sec:model}
We use the following standard conventions. $[k] = \SET{1, \ldots, k}$ is the set of the first $k$ integers. 
Vectors are {bold}. 
When some quantity (such as a frequency/probability $p$ or a set $V$ etc.) has a ground truth value and an estimate by an algorithm, we use $\hatp$ (or $\hatV$ etc.) to distinguish the estimate from the ground truth. 

\subsection{Choice Model}
Let $\universe = \SET{1,2,\ldots,\numel}$ be the universe of items and $\Pi$ be the set of all possible rankings over $\universe$.  That is, $|\Pi| = \numel!$. We consider a non-parametric choice model, where each ranking $\pi\in \Pi$ over $\universe$ represents the preferences of a type. $\pi(i)$ is the \Kth{i} most preferred item by consumers of type $\pi$, and for any item $z \in \universe$, $\pi^{-1}(z)$ is the rank/position of item $z$ for consumers of type $\pi$. For each ranking $\pi\in\Pi$, let $p(\pi)$ denote the probability of a consumer having the ranking $\pi$. Note that for some $\pi\in \Pi$, we can have $p(\pi) =0$.  When a consumer of type $\pi$ is offered an assortment of items $\offerset \subseteq \universe$, she chooses her most preferred item among those in $\offerset$; that is, she chooses the item $z \in \offerset$ such that $\pi^{-1}(z) < \pi^{-1}(x)$ for all $x \in \offerset, x \neq z$. 
We write $\choicetype{z}{\offerset} = \Set{\pi \in \Pi}{\pi^{-1}(z) < \pi^{-1}(x) \text{ for all $x\in \offerset, x \neq z$}}$ for the set of types $\pi\in\Pi$ that choose item $z \in \offerset$ when the set $\offerset$ is offered. Then, for any set $\offerset$ and item $z \in \offerset$, we define $q_z(\offerset)$ as the probability that a consumer chooses item $z$ when the set $\offerset$ is offered. That is, 
\[
q_z(\offerset) 
\; = \; p(\choicetype{z}{\offerset}) 
\; = \; \sum_{\pi\in \choicetype{z}{\offerset}} p(\pi).
\]

\subsection{Frequent vs.~Rare Rankings, and a Generative Model}
Types that only occur very rarely would need {a very large} number of samples to  be estimated accurately. 
Thus, our primary goal is to estimate the DAG on types that occur sufficiently frequently. Specifically,
in our model, there is a set $\freqpi_0$ of (at most) $K$ \emph{candidate frequent types}. In our \emph{general type} model, $\freqpi_0$ can be chosen adversarially. Under the \emph{random type} model, $\freqpi_0$ is chosen as a set of $K$ i.i.d. uniformly drawn rankings from the set of all rankings.\footnote{If there are duplicates among the random $K$ rankings, then $|\freqpi_0| < K$.}
All of our correctness results for inferring types hold in the general type model, while the sample efficiency results hold under the random type model.

An adversary chooses the actual \emph{frequent types} $\freqpi \subseteq \freqpi_0$ from the set of candidate frequent types; all remaining types with positive probability, $\rarepi =\Set{\pi\in\Pi}{p(\pi)>0, \pi\notin\freqpi}$, are called \emph{rare} types.
The adversary then assigns probabilities $p(\pi)$ to each type/ranking $\pi\in \Pi$, subject to the following constraints:
\begin{enumerate}[leftmargin=*]
    \item The frequent types each appear at least a $\kappa$ fraction of the time, i.e., $p(\pi) \geq \kappa$ for all $\pi \in \freqpi$.
    \item The rare types appear less than $\kappa$ fraction of the time, and cumulatively, they appear at most a $\rho$ fraction of the time, i.e., $p(\pi) < \kappa$ for all rare types $\pi \in \rarepi$ and $\sum_{\pi \in \rarepi} p(\pi) \leq \raresum$.
    \item The probabilities define a distribution, i.e., $p(\pi) \geq 0$ for all $\pi$ and $\sum_{\pi \in \Pi} p(\pi) = 1$.
\end{enumerate}

When the candidate frequent rankings $\freqpi_0$ can be adversarial, we call the model with parameters $\kappa \in (0,1)$ and $\raresum \in [0,1)$ the \genmodel{\kappa}{\raresum}; when $\freqpi_0$ is random, we refer to the model as the \randmodel{\kappa}{\raresum}. We note that our  random model captures a scenario where consumer types are very heterogeneous.
%\dkdeletecomment{This kind of belongs more in the Sushi and algorithm discussion than in the definitions.}{We further note that even though the sample complexity parts of our guarantees hold only under the \RANDMODEL, our experiments (in Section~\ref{sec:experiment}) show that the algorithm performs well when applied to the sushi data that contains very similar rankings.}

Notice that due to the lower bound of $\kappa$ on probabilities in $\freqpi$, we have that $|\freqpi| \leq \frac{1}{\kappa}$. Intuitively, we can think of $\freqpi$ as a set of main types in the choice model and $\rarepi$ as the noise. An algorithm's goal will then be to accurately infer the rankings in $\freqpi$ and their probabilities; the larger the combined probability $\raresum$ of the rare types (i.e., the noise), the less accurate the estimates will become.

\subsection{Active Learning of the Choice Model}
We are interested in \emph{actively learning} the choice model in a setting where consumers' types are not observable. In an active learning framework, the algorithm gets to decide --- based on all past observed choices --- which subset/assortment $\offerset \subseteq \universe$ to offer to the next consumer. The algorithm does so without knowing the next consumer's ranking/type, which is drawn from the choice model, independently of all past consumers' types. That is, the consumer is of type $\pi\in\Pi$ with probability $p(\pi)$. Upon being presented with the assortment $\offerset$, the consumer chooses her most preferred item within the set $\offerset$ according to her ranking $\pi$. For simplicity, we assume that a consumer always chooses an item when the assortment offered is non-empty; that is, no type has a no-purchase option.\footnote{When one allows for a no-purchase option in rankings, learning rankings over items that appear after the no-purchase option becomes impossible. For example, assume that, for a consumer type, the no-purchase option is the second most popular item. Then, any active learning algorithm can only learn the most popular item of this consumer type.} We emphasize that due to practical consideration (consider an online retail site in which each consumer purchases items much less frequently than the overall rate of transactions), we assume that each consumer can only be queried once. This prevents an algorithm from immediately collecting detailed information about a type via multiple queries. 

The algorithm's goal is then to learn, with probability at least $1-\delta$, the top $\numitems$ positions of all frequent rankings (i.e., $(\pi(1), \pi(2), \ldots, \pi(\numitems))$, $\pi\in\Pi^F$) and their corresponding probabilities within accuracy $\epsilon$ (plus an error term depending on $\raresum$), using a number of queries which is polynomial in $\numel, 1/\kappa$ and $1/\epsilon${, and $1/\delta$}. (We use the terms ``samples'' and ``queries'' interchangeably throughout the paper.) Here, $\numitems = \alpha \numel$ for some constant $\alpha\in (0,1)$. Note that the algorithm need not learn the rare types; after all, it could take {a large number of} queries to even observe a single sample from the rare types.
We also are only interested in learning the $\numitems$ most popular items in the frequent rankings. Knowing the most popular items could help decision makers with inventory planning and product design decisions, to name a few. Furthermore, learning the top $\numitems$ items can be justified from both practical and technical perspectives. Practically speaking, consumers are often unsure about their preferences for the non-top items in their ranking; hence, they can be inconsistent in choosing among the set of items at the bottom of their ranking (\cite{chernev2006decision}, \cite{goldin2015preference}). In other words, even if the items appearing in low positions were learned, the results might not be very reliable. Moreover, the items at the end of a consumer's ranking are only purchased when none of her top items  are available in the offered sets. Thus, when we see an item in one of the bottom positions of a type being purchased in the transaction data, this data point is likely the result of a different consumer type which ranks the particular item higher. From a technical perspective, as shown in Theorem~\ref{thm:main}, as $\numitems$ gets larger, distinguishing different types of consumers in order to learn their rankings requires a large number of queries. Given the cost of querying, one would like to avoid having large values for $\numitems$, i.e., there is a tradeoff between using few queries vs.~learning the bottom fraction of items of each type.
{A more in-depth discussion of the implications of approximation and truncation can be found in Section~\ref{sec:approximation-implications}.}

\section{Indistinguishable Pairs of Rankings}
\label{sec:indistinguishable}
In this section, we show that when the set of frequent rankings $\freqpi$ contains an \emph{indistinguishable} pair of rankings (in the sense of the following definition), it is information-theoretically impossible to discover the set of types $\Pi^F$ uniquely. This motivates our choice to use DAGs as a representation of the types --- in a sense, they extract the most information that \emph{can} be learned. 

\begin{definition}[Indistinguishability]
\label{def:indistinguish}
Two rankings $\pi$ and $\pi'$ are $i$-indistinguishable for some $2 \leq i \leq \numel-2$ if and only if they satisfy the following three conditions:
\begin{enumerate}[leftmargin=*]
 \item the set of items in the first $i$ position of $\pi$ and $\pi'$ {is the same}, i.e., $\{\pi(j)|j\in [i]\} = \{\pi'(j)| j\in [i]\}$,
 \item at least one item among the top $i$ items has different positions in $\pi$ and $\pi'$, i.e., for at least one $j\in [i],$ we have $\pi(j) \neq \pi'(j)$, and
 \item at least one item among the bottom $\numel-i$ items has different positions in $\pi$ and $\pi'$, i.e., for at least one $j\in\{i+1,\ldots,\numel\},$ we have $\pi(j) \neq \pi'(j)$.
\end{enumerate}

Two rankings are \emph{indistinguishable} if they are $i$-indistinguishable for some $2 \leq i \leq \numel-2$. 
\end{definition}

The following is the main result of this section. 

\begin{restatable}[Impossibility Result]{theorem}{impossibility}
%\begin{theorem}[Impossibility Result] 
\label{thm:impossible}
Suppose that the set of frequent rankings $\Pi^F$ contains two rankings $\pi$ and $\pi'$ that are indistinguishable. Then, it is information-theoretically impossible to discover the set of types $\Pi^F$ uniquely.
%\end{theorem}
\end{restatable}

\textbf{Proof Sketch.}
Given permutations $\pi$ and $\pi'$ that are $i$-indistinguishable for $2 \leq i \leq \numel-2$, where $\numel$ denotes the number of items, we define $\overline{\pi}$ and $\overline{\pi}'$ as follows:
$\overline{\pi}$ is identical to $\pi$ in the first $i$ positions and identical to $\pi'$ in the remaining positions.
$\overline{\pi}'$ is identical to $\pi'$ in the first $i$ positions and identical to $\pi$ in the remaining positions.
Given probabilities $p$, we then construct an alternate choice model $\overline{p}(\cdot)$ where  ${\overline{\pi}}, \overline{\pi}'$ are assigned probabilities similar to $\pi$ and $\pi'$, respectively, with slight adjustments.
Specifically, $\overline{p}(\overline{\pi}) = \overline{p}(\overline{\pi}') = {p(\pi)}$, $\overline{p}(\pi') = {p(\pi')-p(\pi)}$, and $\overline{p}(\pi) = 0$.
For all other permutations $\pi''\in \Pi$, $\overline{p}(\pi'') = p(\pi'')$.
{Using a coupling argument, detailed in Appendix~\ref{sec:proof:thm:impossible}, one can then show that no algorithm can distinguish between the choice models $p$ and ${\overline{p}}$.}
\QED

We remark that even if the frequent types are drawn i.i.d.~uniformly, indistinguishable pairs are likely to occur --- in particular, assuming the random model for frequent types does not obviate the need for DAG representations. To see this, consider $i=2$. Given a ranking $\pi$, a uniformly random ranking starts with $(\pi(2), \pi(1))$ with probability $\frac{1}{\numel (\numel-1)}$, so it is 2-indistinguishable from $\pi$ with probability $\frac{1}{\numel (\numel-1)} \cdot (1-1/(\numel-2)!) \geq \frac{1}{\numel^2}$ (for $\numel \geq 4$, which is necessary for indistinguishability). By a Birthday Paradox argument, a 2-indistinguishable pair of types will occur with constant probability when $K = \Omega(\numel)$, and with high probability when $K = \omega(\numel)$. (Recall that $K$ is an upper bound on the number of frequent types.) Thus, as the number of types will typically exceed the number of items, indistinguishable pairs must be accounted for, even in our random model.

\section{DAG Representation of a Choice Model}
\label{sec:dag-rep}
Theorem \ref{thm:impossible} shows that when there are indistinguishable pairs of (frequent) rankings, no algorithm can learn the choice model. However, even if the set of frequent rankings $\freqpi$ contains pairs of indistinguishable rankings, we would still like to recover as much information as possible. 
We therefore introduce the Directed Acyclic Graph (DAG) representation for a choice model{, which} effectively encodes the maximum information  that can be deduced from a sampling process, in a sense we will elaborate on in Section~\ref{sec:choice_DAG}. 

The motivation for the DAG model arises from the proof of Theorem~\ref{thm:impossible}. Consider two $i$-indistinguishable rankings $\pi, \pi'$. {We} cannot tell apart a world in which just these two types are present, or some/all of the probability lies on types combining the ranking of the first $i$ items according to $\pi$ with the ranking of the remaining $n-i$ items according to $\pi'$ (and vice versa). Let $\toppi = \{ \pi(1), \ldots, \pi(i) \} = \{ \pi'(1), \ldots, \pi'(i) \} $ denote the common set of $i$ items ranked first (in different orders) by $\pi, \pi'$. Then, a sampling-based algorithm \emph{can} infer the probabilities with which $A$ is ranked according to $\pi$ and according to $\pi'$, and similarly for $\universe \setminus A$; it just cannot infer how these rankings are ``combined.'' 
We can consider this as the rankings $\pi, \pi'$ ``merging'' after position $i$, because they both have the set $A$ in the first $i$ positions; subsequently, they split again, but the ``merge point'' corresponds to an indistinguishability.
This intuition generalizes to more complex similarities and differences between types, and leads to our definition of the DAG representation. We will later see in Figure~\ref{fig:onedag} that $\freqpi$ and $\bar{\Pi}^F$ in Theorem \ref{thm:impossible}, with their corresponding types' probability, have the same DAG representation.

\subsection{Edges and Prefixes, and their Probabilities in the DAG Representation}
\label{sec:edges-prefixes}

The DAG representation relies on the notion of \emph{prefixes} in the choice model.
We begin by defining the following nomenclature for prefixes and sets of prefixes.

{
\begin{definition}[{Prefixes, Edges, and Their Probabilities}]
\label{def:prefixes}
\label{def:prefix_edge_prob}
\begin{enumerate}
\item A \emph{prefix} of size $j$ of a ranking $\pi$ is the set of the top (most preferred) $j$ items in type $\pi$, i.e., $\Set{\pi(j')}{j'=1,\ldots,j}$. 

\item The probability of {a prefix $A$ of size $j$} is
the probability that the set of the top-$j$ items of a random consumer type equals $A$, i.e.,
\begin{equation} \label{eq:prob_prefix}
 p(A) = \sum_{\pi: \Set{\pi(j')}{j'=1,\ldots, j}=A} p(\pi).
\end{equation}

\item For $z \notin A$, we write 
 $\edgetype{A}{z} = \Set{\pi}{\SET{\pi(1), \pi(2), \ldots, \pi(j)} = A, \pi(j+1) = z, p(\pi) > 0}$ for the set of types/rankings of positive probability that have the set $A$ in the first $j$ positions (in any order) and the item $z$ in position $j+1$. We then write
\begin{equation} \label{eq:prob_edge}
 \edgeprob{A}{z} = \sum_{\pi \in \edgetype{A}{z}} p(\pi)
\end{equation}
for the combined probability of such types/rankings. We refer to $\edgeprob{A}{z}$ as the \emph{edge probability} between nodes $A$ and $A\cup\SET{z}$.
\end{enumerate}
\end{definition}}

{
The following observations follow immediately from the definitions:
\begin{proposition} \label{prop:basic-properties}
\begin{enumerate}
%\item $p(A) = 0$ whenever $A \notin \familysetnode_{|A|}$. 
\item For fixed $z$, the sets $\edgetype{A}{z}$ of types are disjoint, i.e., $\edgetype{A}{z} \cap \edgetype{A'}{z} = \emptyset$, for any $A,A'$ with $A\neq A'$.
\item The prefix probabilities can be expressed as follows:
\begin{align} 
p(A) = \sum_{z \notin A} \edgeprob{A}{z} = \sum_{z \in A} \edgeprob{A \setminus \SET{z}}{z}, \label{eqn:flow-equation}
\end{align}
where the first representation applies
only when $A$ is not the set of all the items, and the
second representation applies only for non-empty $A$.
\end{enumerate}
\end{proposition}

The {second} part of the proposition {suggests interpreting} the $\edgeprob{A}{z}$ as transition probabilities between prefixes, and motivates the nomenclature of ``edges''.
Indeed, prefixes and edges between them exactly form the primitives of the DAG representation.
}

\subsection{The DAG Representation}
Utilizing the notions of prefixes and edges, we now define the DAG representation of a non-parametric choice model.
\begin{figure}[htb]
    \centering
    \begin{minipage}{.47\textwidth}
        \centering
        \includegraphics[width=\linewidth]{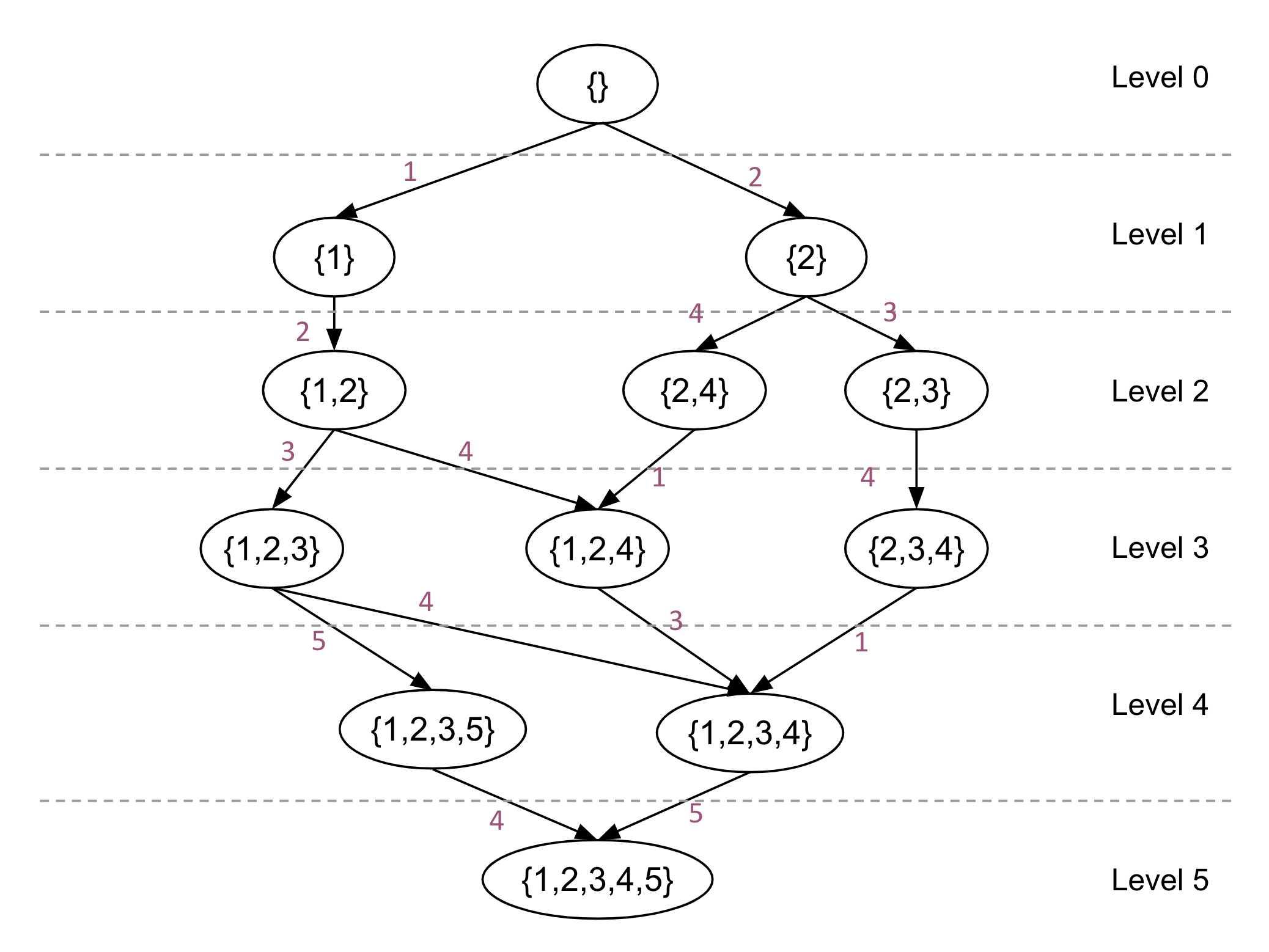}
        \caption{The DAG corresponding to \((\Pi,p)\) with \(p((1,2,3,4,5)) = p((1,2,3,5,4)) = p((1,2,4,3,5)) = p((2,3,4,1,5)) = p((2,4,1,3,5)) = 0.2\).}
        \label{fig:dag_example}
    \end{minipage}\hfill
    \begin{minipage}{.45\textwidth}
        \centering
        \includegraphics[width=\linewidth]{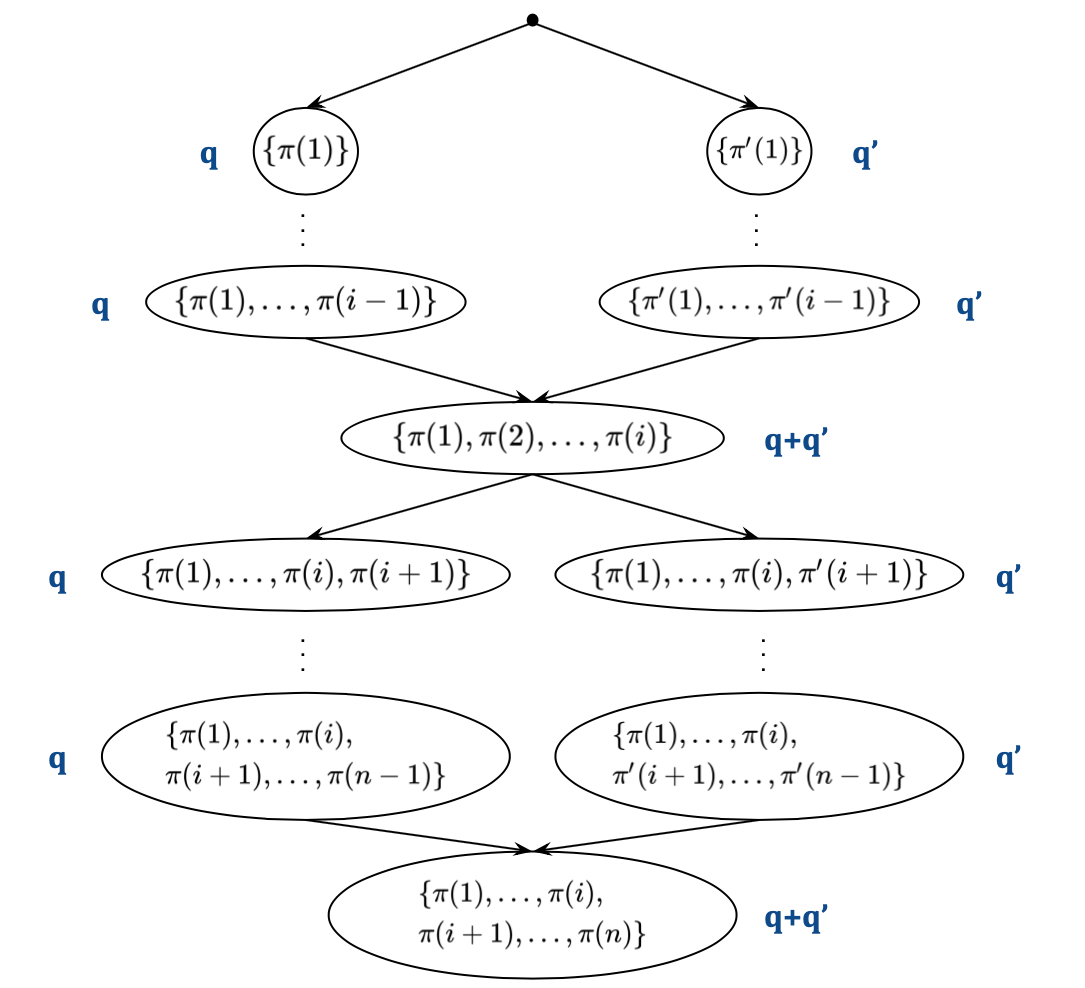}
        \caption{The DAG corresponding to the indistinguishable pairs of rankings from Theorem~\ref{thm:impossible}.}
        \label{fig:onedag}
    \end{minipage}
\end{figure}

\begin{definition}[DAG Representation of $(\Pi,\bf{p})$]
 {
Let $V = {\Set{A}{p(A) > 0}}$ be the set of all prefixes $A$ with positive probability, and $E = \Set{(A, A \cup \SET{z})}{\edgeprob{A}{z} > 0}$ the set of all edges with positive probability.

The Directed Acyclic Graph (DAG) representation of the choice model $(\Pi,p)$, denoted by $G = (V,E, {(p(A))_{A \in V}, (\edgeprob{A}{z})_{(A, A \cup \SET{z}) \in E}})$, has node set $V$ and edge set $E$.
The probabilities of nodes $A$ and edges $(A, A \cup \SET{z})$ are $p(A)$ and $\edgeprob{A}{z}$, respectively.
We refer to {the prefixes $A$ with $|A| = j$} as \emph{level $j$} of the graph\footnote{Note that edges exist solely between consecutive levels.}, and to $z$ as the \emph{label} of the edge $(A, A {\cup} \SET{z})$. 

Two DAG representations $G = (V, E, \bm{p}, \bm{e})$ and $G' = (V', E', \bm{p}', \bm{e}')$  are \emph{identical} if and only if $V = V'$, $E = E'$, $p(A) = p'(A)$ for all $A \in V$, and $\edgeprob{A}{z} = e'_{A \oplus z}$ for all $(A, A \cup \SET{z}) \in E$.}
\end{definition}

An example DAG illustrating this definition is shown in Figure~\ref{fig:dag_example}, for the choice model $(\Pi,\bm{p})$ whose types with positive probability are
%\begin{align}
$\SET{(1,2,3,4,5), (1,2,3,5,4), (1,2,4,3,5), (2,3,4,1,5), (2,4,1,3,5)}$, %\label{eq:Pi}
%\end{align}
each with probability $0.2$. In this example, the nodes on level one are $\SET{1}$ and $\SET{2}$ (the top elements under $\Pi$), the nodes on level two are $\SET{1, 2}, \SET{2, 3}$, and $\SET{2, 4}$ (the top pairs of elements under $\Pi$), and the nodes on level three are $\SET{1,2,3}, \SET{1,2,4}$, and $\SET{2,3,4}$ (the top triples). 
This example also shows that multiple edges can have the same label. For example, both the edge $(\SET{1}, \SET{1,2})$ and $(\emptyset, \SET{2})$ are labeled $2$. 
{For a second example, see} Figure~\ref{fig:onedag}, which shows the DAG representation of $\freqpi$ and $\bar{\Pi}^F$ from Theorem \ref{thm:impossible}. There, it can be observed that $\freqpi$ and $\bar{\Pi}^F$ have the same DAG representation, and that the DAG merges at level $i$ and splits at level $i+1$.

We finish this section by defining the notion of truncated DAGs, which are useful for describing the intermediate stages of the DAG construction algorithm.

{\begin{definition}[Truncated DAG]
\label{def:truncated-DAG}{
For any DAG model $G = (V, E, \bm{p}, \bm{e})$, given an integer $j \in \SET{0, 1, \ldots, n}$, we let $G_j$ denote the DAG (with associated probabilities) restricted to levels $0, \ldots, j$.
Formally, we let $V_j = \Set{A \in V}{|A| \leq j}$ be the set of prefixes in $V$ of at most $j$ items, $E_j = E[V_j \times V_j]$ the set of edges induced by these prefixes, and $\bm{p}_j = (p(A))_{A \in V_j}$ and $\bm{e}_j = (\edgeprob{A}{z})_{A \in V_j, A \cup \SET{z} \in V_j}$.
Then, $G_j = (V_j, E_j, \bm{p}_j, \bm{e}_j)$.}
\end{definition}}

\section{DAGs and Choice Probabilities} 
\label{sec:dag-choice}
{
We first show that the DAG representation captures all the information inherent in the choice probabilities, in the sense that the choice probabilities can be fully (and efficiently) recovered from the corresponding DAG representation.
In turn, given oracle access to the (exact) choice probabilities, the DAG representation of a choice model can be constructed in time polynomial in the number of types.
The proof of the latter result also serves as an easy warm-up, illustrating the central underlying ideas that we significantly refine to adaptively deal with sampled choice frequencies, rather than exact choice probabilities.

{
\subsection{Computing the Choice Probabilities from a given DAG} \label{sec:choice_DAG}

Given an item set $S$ and item $z \in S$, we would like to compute the probability $q_z(S)$ of a random type choosing $z$ when offered the set $S$.
The calculation of choice probabilities is captured by the following proposition.

\begin{restatable}[Computing Choice Probabilities from a DAG]{proposition}{constructingchoice}
%\begin{proposition}[Computing Choice Probabilities from a DAG]
\label{prop:exact_choice_prob}
Let $(\Pi, \bf{p})$ be a choice model, and $(V,E,\bf{p},\bf{e})$ a DAG representation of $(\Pi,\bf{p})$.
Let $S \subseteq \universe$ be an item set and $z \in S$.
Then, the set of types choosing $z$ from $S$ and its associated probability are characterized by the following:
\begin{align*}
  \choicetype{z}{S} & = \bigcup_{A: A \cap S = \emptyset} \edgetype{A}{z} 
  & 
  q_z(S) & = \sum_{A: A \cap S = \emptyset} \edgeprob{A}{z}.
\end{align*}
%\end{proposition}
\end{restatable}

Proposition~\ref{prop:exact_choice_prob}{, which is shown in {Appendix}~\ref{sec:proof:prop:exact_choice_prob},}  implies a straightforward linear-time algorithm for computing $q_z(S)$: iterate through all prefixes $A \in V$, and if $A \cap S = \emptyset$, then add $\edgeprob{A}{z}$ (which might be 0, if the edge does not exist) to the total.
}

{
\subsection{Constructing the DAG with an Exact Choice Probability Oracle}
\label{sec:construct_DAG_exact}

In this section, we show that when an algorithm has access to the exact choice probabilities --- as opposed to empirical and noisy choice frequencies --- for arbitrary sets and items, it can reconstruct the DAG efficiently. We model such access by an oracle which, given $S$ and $z$, will correctly return $q_z(S)$.

The algorithm (Algorithm~\ref{alg:all_exact}) constructs the DAG level by level. 
Recall that throughout the paper, we denote quantities which the algorithm computes by a hat, to distinguish them from the (possibly different) ground truth values.
The algorithm starts with level 0, composed of only the node $\emptyset$ with $\hatp(\emptyset) = 1$.
In each iteration $j = 0, 1, \ldots, n-1$, for each node $A$ on level $j$, it determines all outgoing edges $(A, A \cup \SET{z})$ {of positive probability} using Lemma~\ref{lem:inter_2} below.
Then, level $j+1$ is composed of all nodes $A \cup \SET{z}$ that were the head of at least one of these edges. 
The probability associated with a node $A$ on level $j+1$ can be obtained from Proposition~\ref{prop:basic-properties} as 
%  \[
$p(A) = \sum_{z \in A} \edgeprob{A \setminus \SET{z}}{z}$.
%  \]

\begin{lemma} \label{lem:inter_2} 
  For any prefix $A$ and $z \notin A$, we have 
%\begin{equation}
%\label{eq:q_interference}
$\edgeprob{A}{z} =  q_z(\universe\setminus A) -  \sum_{A' \subsetneq A} \edgeprob{A'}{z}$.
%\end{equation} 
\end{lemma}

\begin{proof}
{Proposition~\ref{prop:exact_choice_prob} with $S = \universe \setminus A$ gives
  $q_z(\universe \setminus A) = \sum_{A': A' \cap (\universe \setminus A) = \emptyset} \edgeprob{A'}{z} = \sum_{A' \subseteq A} \edgeprob{A'}{z}$.
  Solving for $\edgeprob{A}{z}$ proves the lemma.
  }  
\end{proof}

\begin{algorithm}
\footnotesize{
\caption{\footnotesize{Constructing the Exact DAG using {Oracle Access to Exact} Choice Probabilities}}\label{alg:all_exact}

\KwIn{An exact oracle for the choice probabilities of a non-parametric choice model {$(\Pi,\bf{p})$}.}

\KwOut{The (exact) DAG representation {$G$} of the choice model and its associated probabilities.}

%\tcc{Initialization}
  
  Initialize $G_{0} { = (\SET{\emptyset},\emptyset)}$ {with only one node for the empty set, and no edges}. Set $\hatp(\emptyset) = 1$. 

%\tcc{Constructing the DAG Level by Level}

\For{$j=0,\ldots,n-1$}{
  Initialize $V_{j+1}$ and $E_{j+1}$ as empty sets.
  
  \For{each set $A$ on level $j$ of $G_j$}{
    \For{each item $z \notin A$}{
      Set $\hatedgeprob{A}{z} =  q_z(\universe\setminus A) -  \sum_{A' \subsetneq A} \hatedgeprob{A'}{z}$. \tcc*{Compute the edge probabilities}
       
     \If{{$\hatedgeprob{A}{z} > 0$}}{
       \If{node $A \cup \SET{z}$ is not in $V_{j+1}$}{
            Add node $A \cup \SET{z}$ to $V_{j+1}$. Set $\hatp(A \cup \SET{z}) = 0$.
          }
        
          Set $\hatp(A \cup \SET{z}) = \hatp(A \cup \SET{z}) + \hatedgeprob{A}{z}$. 
          
          Add a directed edge from $A$ to $A \cup \SET{z}$ in $E_{j+1}$, labeled with $z$, and with probability $\hatedgeprob{A}{z}$}.
       }          
    }
  }
}
\end{algorithm}

{
The performance of Algorithm~\ref{alg:all_exact} is captured by the following theorem, {whose proof is in Appendix}~\ref{sec:proof:thm:construct_exact_DAG}.

\begin{restatable}[Constructing the Exact DAG using Exact Choice Probabilities]{theorem}{constructingexact}
%\begin{theorem}[Constructing the Exact DAG using Exact Choice Probabilities] 
\label{thm:construct_exact_DAG}
Let $(\Pi,\bf{p})$ be the underlying choice model with $T$ types, 
and assume that Algorithm~\ref{alg:all_exact} has exact oracle access to the choice probabilities.
Then, Algorithm~\ref{alg:all_exact} runs in time $O(n^2 T)$ and outputs a correct DAG representation of $(\Pi,\bf{p})$.
%\end{theorem}
\end{restatable}
}
}

}

\section{Active Learning of {a} DAG using {Estimated} Choice Frequencies}
\label{sec:dag-algo}
{

We presented Algorithm~\ref{alg:all_exact} and its analysis  
under the assumption of exact access to choice probabilities. In reality, information about consumers' choice behavior will virtually always be inferred from observations, and as such be based on \emph{empirical frequencies}, which will differ subtly from exact choice probabilities.

One can of course still use Algorithm \ref{alg:all_exact} to construct a DAG, using empirical estimates of the quantities $q_z(\universe \setminus A)$.
However, this will lead to corresponding misestimates of the quantities $\edgeprob{A}{z}$.
Importantly, the key equation from Lemma~\ref{lem:inter_2} for computation of edge probabilities, $\edgeprob{A}{z} = q_z(\universe \setminus A) - \sum_{A' \subsetneq A} \edgeprob{A'}{z}$, uses previously computed edge probability estimates in later iterations.
As such, it is susceptible to an accumulation of errors, which might amplify errors.
That this can indeed occur is illustrated in the following example, depicted in Figure~\ref{fig:ex_DAG_inter}:
}

{ 
\begin{example}
\label{example:naive_2} 
Consider a scenario with eight items and the following frequent rankings:\\
$\Pi = \{(1,2,3,4,5,6,7,8), (1,2,3,4,5,6,8,7), (1,2,3,4,7,5,6,8), (1,2,3,5,7,4,6,8), (1,2,3,6,7,4,5,8),$\\$ (1,7,2,3,4,5,6,8), (2,7,1,3,4,5,6,8),(3,7,1,2,4,5,6,8), (7,1,2,3,4,5,6,8)\}$.

The goal is to estimate the edge probability $\edgeprob{\SET{1,2,3,4,5,6}}{7}$, corresponding to the bolded node and edge on the \Kth{7} level of the DAG, as highlighted in purple in Figure~\ref{fig:ex_DAG_inter}. The proper subsets $A' \subsetneq A$ with an outgoing edge labeled $z = 7$ are $\SET{1,2,3,4}, \SET{1,2,3,5}, \SET{1,2,3,6}, \SET{1}, \SET{2},\SET{3}, \emptyset$, highlighted in blue.

To compute  the edge probability $\edgeprob{\{1,2,3,4,5,6\}}{7} = \edgeprob{A}{z}$, {one can repeatedly apply Lemma~\ref{lem:inter_2} to all terms of the form $\edgeprob{A'}{z}$ and rearrange, obtaining that} 
\begin{equation} \label{eq:big_coef_1}
\begin{aligned}
      \sum_{A' \subsetneq A} \edgeprob{A'}{z}
   & = 
     \edgeprob{\SET{1,2,3,4}}{7} +  \edgeprob{\SET{1,2,3,5}}{7} + \edgeprob{\SET{1,2,3,6}}{7}+\edgeprob{\SET{1}}{7} + \edgeprob{\SET{2}}{7} + \edgeprob{\SET{3}}{7} + \edgeprob{\emptyset}{7}
 %  \\ & \overset{(1)}{=} 
 %     q_7\left(\universe\setminus\SET{1,2,3,4}\right) - \edgeprob{\SET{1}}{7} -\edgeprob{\SET{2}}{7} - \edgeprob{\SET{3}}{7} - \edgeprob{\emptyset}{7}
 %    \\ &\quad + q_7\left(\universe\setminus\SET{1,2,3,5}\right) - \edgeprob{\SET{1}} {7} -\edgeprob{\SET{2}}{7} - \edgeprob{\SET{3}}{7} - \edgeprob{\emptyset}{7}
 %    \\ &\quad + q_7\left(\universe\setminus\SET{1,2,3,6}\right) - \edgeprob{\SET{1}}{7} -\edgeprob{\SET{2}}{7} - \edgeprob{\SET{3}}{7} - \edgeprob{\emptyset}{7}+\edgeprob{\SET{1}}{7} +\edgeprob{\SET{2}}{7} + \edgeprob{\SET{3}}{7} + \edgeprob{\emptyset}{7}
 % \\ & \overset{(2)}{=}
 %    q_7\left(\universe\setminus\SET{1,2,3,4}\right) + q_7\left(\universe\setminus\SET{1,2,3,5}\right) + q_7\left(\universe\setminus\SET{1,2,3,6}\right)  
 %  \\ &\quad -
 %    2\edgeprob{\SET{1}}{7} - 2\edgeprob{\SET{2}}{7} - 2\edgeprob{\SET{3}}{7} - 2\edgeprob{\emptyset}{7}
 % \\ &\overset{(3)}{=} 
 %    q_7\left(\universe\setminus\SET{1,2,3,4}\right) + q_7\left(\universe\setminus\SET{1,2,3,5}\right) + 
 %    q_7\left(\universe\setminus\SET{1,2,3,6}\right) 
 %    \\ &\quad 
 %    - 2(q_7\left(\universe\setminus\SET{1}\right) - \edgeprob{\emptyset}{7})
 %    - 2(q_7\left(\universe\setminus\SET{2}\right) - \edgeprob{\emptyset}{7})
 %    - 2(q_7\left(\universe\setminus\SET{3}\right) - \edgeprob{\emptyset}{7})
 %    - 2\edgeprob{\emptyset}{7}
\\ & =
    q_7\left(\universe\setminus\SET{1,2,3,4}\right) + q_7\left(\universe\setminus\SET{1,2,3,5}\right) + q_7\left(\universe\setminus\SET{1,2,3,6}\right)
    \\ &\quad 
    - 2q_7\left(\universe\setminus\SET{1}\right) - 2q_7\left(\universe\setminus\SET{2}\right) - 2q_7\left(\universe\setminus\SET{3}\right) 
    + 4q_7\left(\universe\setminus\emptyset\right).
%&\overset{=} q_7\left(\universe\setminus\{1,2,3,4\}\right) + %q_7\left(\universe\setminus\{1,2,3,5\}\right) + 
%q_7\left(\universe\setminus\{1,2,3,6\}\right)\\
%&\quad - 2q_7\left(\universe\setminus\{1\}\right) - 2q_7\left(\universe\setminus\{2\}\right) -
%2q_7\left(\universe\setminus\{3\}\right) + 
% 4q_7\left(\universe\setminus\emptyset\right)\,.
\end{aligned}
\end{equation}

Observe that the coefficients grow exponentially in the number of {times a set occurred in the expansion} steps, as observable in the coefficient {$q_7(\universe \setminus \emptyset)$} with the largest absolute value, i.e., $4$. This exponential growth pattern persists for a generalization of this example to instances with $(3k-1)$ items and $3k$ types, where the biggest coefficient has an absolute value of $2^{k-1}$.
\end{example}}
%\lipsum[1-2]

%\lipsum[3-4]

Having exponentially large coefficients is not appealing, because they exponentially amplify any estimation errors due to rare rankings. Counteracting such estimation errors {with improved estimation accuracy} would then require exponentially many samples. 
Our main technical contribution is a more careful way to select the queries to use, based on approximate solutions to Set Cover instances that we will introduce shortly. This approach guarantees that the resulting linear system has a good condition number. While the number of distinct queries required is exponential in the size of the Set Cover solution, whenever this solution has only logarithmic size, the resulting number of queries is polynomial. In particular, this is the case under the random model.

\begin{wrapfigure}{R}{0.43\textwidth}
    \centering
    \includegraphics[width=0.4\textwidth]{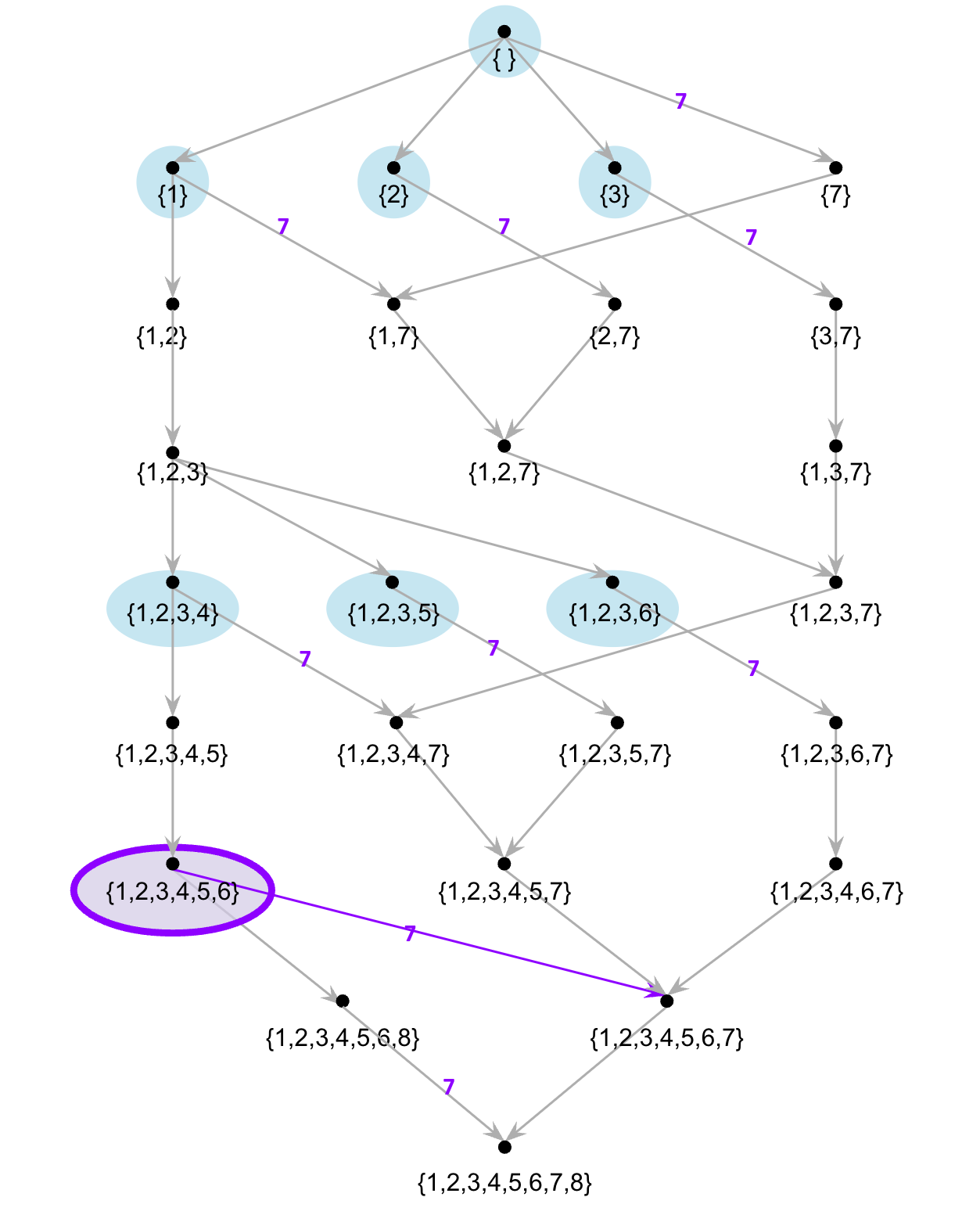} % Adjust the width slightly smaller to leave space for text
    \caption{\footnotesize{The DAG corresponding to Example~\ref{example:naive_2}. The rankings corresponding to the blue nodes interfere when trying to estimate the probability of the ranking corresponding to the purple node and edge. That is, the blue nodes represent the prefixes $A' \subsetneq A = \{1, 2, 3, 4, 5, 6\}$ in Lemma~\ref{lem:inter_2}.}}
    \label{fig:ex_DAG_inter}
\end{wrapfigure}

{
\subsection{{The Active Learning Algorithm}}

In order to avoid the possibly exponential error accumulation, our adaptive algorithm \AlgDAG --- given as Algorithm~\ref{alg:all} --- modifies Algorithm~\ref{alg:all_exact} in three ways:
\begin{enumerate}[leftmargin=*]
    \item To compute the edge probabilities $\edgeprob{A}{z}$, rather than using Lemma~\ref{lem:inter_2}, \AlgDAG uses a more complex new algorithm \AlgIE (Algorithm~\ref{alg:ie}) based on the Inclusion-Exclusion Principle.
    This algorithm internally uses random samples from the population for estimating choice frequencies, and thus must be given error parameters $(\epsilon, \delta)$. It will adjust the number of samples to ensure that with probability at least $1-\delta$, the estimate is accurate to within error $\epsilon$.
    \item Rather than including all edges whose estimated probability is positive, \AlgDAG only includes edges whose estimated probability is at least $\kappa/2$; i.e., it focuses on edges with sufficiently high probability of corresponding types.
    The reason for excluding edges with small probability estimates is that the their estimates are likely very inaccurate unless a large number of samples is used. 
    As a result, \AlgDAG produces a DAG only for the frequent types.
    \item \AlgDAG only builds the first $\numitems = \alpha \cdot n$ levels of the DAG, for a user-specified parameter $\alpha$.
    The reason is that again, estimates for the lower levels tend to become less accurate.
\end{enumerate}

Apart from these modifications, \AlgDAG is identical to Algorithm~\ref{alg:all_exact}.
{It} obtains as input the parameters $\alpha, \epsilon, \delta$, in addition to the parameters $\kappa, \raresum$ of the choice model (or a valid lower bound thereof).
}

\begin{algorithm}[htb]\footnotesize{
\caption{\footnotesize{Adaptive DAG Ranking Algorithm (\AlgDAG)} \label{alg:all}}

\KwIn{$\alpha$ %fraction of positions to estimate 
$\epsilon$% desired bound on deviations in output estimates
, and $\delta$. %: maximum allowed probability of large deviations
}

\KwOut{
A DAG representation $(\hatG_{\numitems},\hatpv,\hatev)$ estimating the top $\numitems = \alpha n$ items in the set of rankings $\Pi^F$ with their associated probabilities}

%\tcc{Initialization}
  Initialize $\hatG_{0} = (\SET{\emptyset},\emptyset)$ as a graph with only one node for the empty set, and no edges.  Set $\hatp(\emptyset) = 1$. 
  
%\tcc{Estimating the DAG Level by Level}

\For{$j=0,\ldots,\numitems-1$}{
  Initialize $\hatV_{j+1}$ and $\hatE_{j+1}$ as empty sets.
  
  \For{each set $A$ on level $j$ of $\hatG_j$}{
    \For{each item $z \notin A$}{
      %\tcc*{Estimate the edge probabilities  via \AlgIE}:
      Set $\hatedgeprob{A}{z}$ to $\text{\AlgIE}({\hatG}_j,A,z;\alpha,\epsilon',\delta')$ with  $\epsilon'=\min(\epsilon/\numitems, \kappa/4)$ and $\delta'={\frac{\delta}{\alpha \numel^2 K}}$.
      
      \If{$\hatedgeprob{A}{z} \geq \kappa/2$}{
        \If{node $A \cup \SET{z}$ is not in $\hatV_{j+1}$}{
          Add node $A \cup \SET{z}$ to $\hatV_{j+1}$. Set $\hatp(A \cup \SET{z}) = 0$.
        }
        
        Set $\hatp(A \cup \SET{z}) = \hatp(A \cup \SET{z}) + \hatedgeprob{A}{z}$. 
        
        Add a directed edge from $A$ to $A \cup \SET{z}$ in $\hatE_{j+1}$, labeled with $z$, and with probability ${\hatedgeprob{A}{z}}$.
      }
    }
  }
}}
\end{algorithm}

{
\subsection{The Inclusion-Exclusion Approach}
\label{sec:AlgIE_new}
As we observed in Example~\ref{example:naive_2} with a coefficient of 4 --- and asserted to hold more generally --- applying Lemma~\ref{lem:inter_2} repeatedly in the straightforward way can ultimately express the edge probabilities as a linear combination of choice frequencies in which the coefficients may grow exponentially. This exponential growth amplifies small estimation errors, and would thus require extremely accurate frequency estimates.
It is desirable to instead express the edge probabilities as a linear combination of choice frequencies in a way that keeps the coefficients bounded. We will use the Inclusion-Exclusion formula to produce a linear combination in which all coefficients are $\pm 1$. 

To motivate the approach, we revisit Example~\ref{example:naive_2}, depicted in Figure~\ref{fig:ex_DAG_inter}.
There, the goal was to estimate the probability of the edge labeled 7 emanating from the purple node $\SET{1,2,3,4,5,6}$.
To do so, we considered all nodes with outgoing edges labeled 7 which are ancestors of the node $\SET{1,2,3,4,5,6}$.
These are the blue nodes $\emptyset, \SET{1}, \SET{2}, \SET{3}, \SET{1,2,3,4}, \SET{1,2,3,5}, \SET{1,2,3,6}$. 
We observe that each type corresponding to one of these prefixes prefers item 7 over at least one of $\SET{5,6}$ (in addition to item 8, which 7 is preferred over by all ancestors of $\SET{1,2,3,4,5,6}$, by virtue of not being in the set $\SET{1,2,3,4,5,6}$).
As an example, all types corresponding to $\SET{1,2,3,5}$ prefer both 7 and 8 over 6; see Figure \ref{fig:ex_DAG_inter}.
This observation ``almost'' allows us to write 
$\sum_{A' \subsetneq \SET{1,2,3,4,5,6}} \edgeprob{A'}{z}
= q_7(\SET{6, 7, 8}) + q_7(\SET{5, 7, 8})$.
The ``almost'' is because simply adding these two probabilities  double-counts types who prefer 7 over \emph{both} 5 and 6 (in addition to 8). To correct for the double-counting, we subtract the probability of double-counted types, which is $q_7(\SET{5, 6, 7, 8})$. 
We thus obtain that 
%\[ 
$
\sum_{A' \subsetneq \SET{1,2,3,4,5,6}} \edgeprob{A'}{z}
= q_7(\SET{6, 7, 8}) + q_7(\SET{5, 7, 8}) - q_7(\SET{5,6,7,8}).
$
%\]

Notice that in this formula, all coefficients are $\pm 1$.
Inspecting the sets $\SET{6,7,8}, \SET{5,7,8}$ more closely, we notice that their complements ($\SET{1,2,3,4,5}, \SET{1,2,3,4,6}$) are direct predecessors of $\SET{1,2,3,4,5,6}$, and each of the ancestors of interest (blue nodes) is also an ancestor of at least one of $\SET{1,2,3,4,5}, \SET{1,2,3,4,6}$. In this sense, these two sets \emph{cover} all blues nodes, and this notion lies at the heart of our improved estimation approach. Formally, we define the following:
}

{
\begin{definition}[Set Cover]
\label{def:set-cover}
   Fix a set $A$ and element $z \notin A$.
   Let 
   \[
   \familyminus(A) = \Set{A' \subsetneq A}{|A'| = |A| - 1}
   \]
   denote the collection of subsets of $A$ with size $|A| - 1$.
   Let $\familysetnodepre(A, z) 
   = \Set{A' \subsetneq A}{\edgeprob{A'}{z} > 0}$ 
   be the set of all prefixes $A' \subsetneq A$ which have an outgoing edge to the node $A' \cup \SET{z}$, and let $\familysetnodepre \subseteq \familysetnodepre(A,z)$.

   We say that $\mathcal{C}$ is a \emph{set cover of $\familysetnodepre$} if (1) $\mathcal{C} \subseteq \familyminus(A)$, and (2) for each set $A' \in \familysetnodepre$, there exists a set $C \in \mathcal{C}$ with $A' \subseteq C$, i.e., each set in $\familysetnodepre$ is contained in some set in $\mathcal{C}$. 
\end{definition}

   Definition~\ref{def:set-cover} is closely related to the standard notion of set covers in algorithm design, a fact we elaborate on more in Appendix~\ref{sec:greedy}, and which we exploit in our algorithm design.
}

{
When $\mathcal{C}$ is a set cover of $\familysetnodepre(A, z)$, the Inclusion-Exclusion formula we previously saw by example generalizes as follows:
\begin{proposition}[Inclusion-Exclusion with Exact Probabilities]
\label{prop:incl-excl-exact}
  Fix a set $A$ and item $z \notin A$, and let $\mathcal{C}$ be a set cover of $\familysetnodepre(A, z)$ using $\familyminus(A)$. 
  Then, 
   \[
   \edgeprob{A}{z}
   = q_z(\universe\setminus A) - \sum_{\mathcal{B} \subseteq \mathcal{C}, \mathcal{B} \neq \emptyset} (-1)^{|\mathcal{B}| + 1} \cdot q_z(\universe \setminus \bigcap_{C \in \mathcal{B}} C).
   \]
\end{proposition}}

{Proposition~\ref{prop:incl-excl-exact}, which is proven  in Appendix \ref{sec:Proposition:proof}, serves as the central motivation for the Inclusion-Exclusion formula used in our algorithm \AlgIE (Algorithm~\ref{alg:ie}).}

\begin{algorithm}[htb]
\footnotesize{
\caption{\footnotesize{The Set Cover based Inclusion-Exclusion Algorithm ($\text{\AlgIE}(\hatG_j,A,z;\epsilon,\delta)$)}}
\label{alg:ie}

\KwIn{Graph $\hatG_j=(\hatV_j, \hatE_j)$, node/set $A$, item $z$, $\epsilon$, and $\delta$}

\KwOut{An estimate of the edge probability $\hatedgeprob{A}{z}$}

\tcc{{Find} a minimum set {cover}}

Let $\familysetnodeprealg = \Set{A' \in \hatV_j}{A'\subsetneq A, A' \cup \SET{z} \in \hatV_j}$ and $\familyminus(A) = \Set{A' \subsetneq A}{|A'|= |A|-1}$.

Find an approximately optimal set cover $\coveralg = %\famsetmingreedy(A,z) = 
\greedymincover\left({\familysetnodeprealg},\familyminus{(A)}\right)$. (See Appendix~\ref{sec:greedy}.)

\tcc{Estimate the probabilities for all subsets in Inclusion-Exclusion}

Let $\numsamples = \left\lceil \frac{2^{2|\coveralg|-1} \cdot (\ln(1/\delta) + (|\coveralg|+1) \cdot \ln(2))}{\epsilon^2} \right\rceil$.

Offer the assortment $\universe \setminus A$ to $\numsamples$ consumers. Let $\hatq_z(\universe \setminus A)$ be the fraction of consumers who choose $z$ under assortment $\universe \setminus A$.

\For{every $\mathcal{B} \subseteq \coveralg,
\mathcal{B} \neq \emptyset$}{
 Offer the assortment $\universe\setminus \bigcap_{B\in\mathcal{B}} B$ to $\numsamples$ consumers.

 Let $\hatq_z(\universe\setminus \bigcap_{B\in\mathcal{B}} B)$ be the fraction of consumers who choose $z$ under assortment $\universe \setminus \bigcap_{B\in\mathcal{B}} B$.
} 
Let $\hatedgeprob{A}{z} = \hatq_z(\universe \setminus A) - \sum_{\mathcal{B}\subseteq \coveralg, \mathcal{B} \neq \emptyset}
(-1)^{|\mathcal{B}|+1} \cdot \hatq_z(\universe\setminus \bigcap_{B\in\mathcal{B}} B)$.

\Return{the edge probability $\hatedgeprob{A}{z}$.}}
\end{algorithm}

{Because the right-hand side of the formula in Proposition~\ref{prop:incl-excl-exact}, which is used to calculate the estimated edge probabilities, contains a number of terms exponential\footnote{Note, however, that some (or many) of the intersections may be empty, so the actual contributing number may be smaller.} in $|\mathcal{C}|$, it is desirable to make this set cover $\mathcal{C}$ as small as possible. 
Unfortunately, as we show in Appendix~\ref{sec:greedy}, our Set Cover problem subsumes the standard \textsc{Set Cover} problem, which is well known to be NP-hard.
Thus, minimizing $|\mathcal{C}|$ exactly is unlikely to be possible in polynomial time. 
Instead, we use the greedy algorithm, which is known (see, e.g., \citep{kleinberg2006algorithm}) to offer an approximation guarantee logarithmic in the number of items.}

{The greedy algorithm for \textsc{Set Cover}, denoted by $\greedymincover\left(\familysetnodepre,\familyminus(A)\right)$, repeatedly selects a set from $\familyminus(A)$ that contains as subsets the largest number of prefixes $A' \in \familysetnodepre$ not contained in any selected set so far. The algorithm terminates when all prefixes in $\familysetnodepre$ have been covered.
See Appendix~\ref{sec:greedy} for details. 

Combining these ideas, we obtain the improved Set Cover based Inclusion-Exclusion Algorithm, \AlgIE (Algorithm~\ref{alg:ie}). Recall that this algorithm is used in \AlgDAG as a way to obtain improved edge probability estimates. Moreover, the number of samples required by \AlgIE depends on the size of the set covers found in the algorithm; hence, if all set covers encountered when \AlgIE is called by \AlgDAG are small enough, the number of samples needed is small.}

\subsection{Guarantees for \AlgDAG}
The key guarantees for \AlgDAG (Algorithm~\ref{alg:all}) are summarized by the following theorem, which we prove in {Appendix}~\ref{sec:proof}. 
In the statement of the theorem, recall that $G^F_{\numitems}$ denotes the subgraph of the true DAG induced by the first $\numitems$ layers and only the prefixes appearing for frequent types.

\begin{theorem}[Performance Guarantee of \AlgDAG]
\label{thm:main} Let $\numitems=\alpha\numel$ for some constant $\alpha\in(0,1)$, and assume that the total probability of the rare rankings is at most $\raresum < \kappa/4$. Under the \genmodel{\kappa}{\raresum}, Algorithm~\ref{alg:all} (run with parameters $\epsilon > 0$ and $\delta\in (0,1)$) returns an estimated choice model $(\hatG_{\numitems}, \hatpv, \hatev)$,
such that with probability at least $1-\delta$, the following all hold:
\begin{align*}
% \hatG_{\numitems} & = G^F_{\numitems} 
\hatV_{\numitems} & = V^F_{\numitems}
& \hatE_{\numitems} & = E^F_{\numitems}
%\quad \text{and} \quad 
& 
\max_{A \in V_{\numitems}^{F}} \left|\hatp(A) - p(A)\right| & \leq \epsilon + \numitems \raresum
& 
\max_{A \in V_{\numitems}^F, A \cup \SET{z} \in V_{\numitems}^F} \left| \hatedgeprob{A}{z} - \edgeprob{A}{z} \right| \leq \epsilon/\numitems + \raresum.
\end{align*}

Furthermore, under the \randmodel{\kappa}{\raresum}, writing $\cons = \log_{\alpha}(1/2) > 0$, with probability at least $1-2\delta$, the total number of queries and the total computation time are bounded by 
\[ O \left( \frac{K^{1+4\cons} \cdot \numel^{4+2\cons} \cdot 
\log (\numel K/\delta)}{\delta^{2\cons} \cdot \min(\epsilon,\kappa)^2}\right).
\]
\end{theorem}

%\ngcomment{add this comment: in practice, we choose $m$ and given this value, we have a trade-off between $\epsilon$ and $\delta$. Smaller $\delta$ leads to a higher probability result at the cost of having a lower accuracy parameter  $\epsilon$. } 

%\ngcomment{I did not make this comment because it is not entirely true. Note that in \AlgIE, the number of samples $\numsamples = \left\lceil \frac{2^{2|\ngedit{\coveralg}|-1} \cdot (\ln(1/\delta) + (|\ngedit{\coveralg}+1) \cdot \ln(2))}{\epsilon^2} \right\rceil$ depends on the size of the set cover $\coveralg$, which keeps changing during the run of the algorithm. So, I cannot map $m$ to a pair of $(\epsilon, \delta)$.}

\AlgDAG works for both the General and Random Models. However, the size of all set covers is \emph{guaranteed} to be logarithmically small only for the Random Model (with high probability) --- this results in a polynomial total number of samples in that case. Nevertheless, even for the General Model, if the size of all set covers is small, the required number of samples is also small. Furthermore, notice that in \AlgIE, after finding the approximate minimum set cover, the algorithm knows the size of the set cover. Therefore, \AlgIE (and consequently also \AlgDAG) can easily diagnose when a larger number of samples is needed to obtain good accuracy; alternatively, the algorithm can use a theoretically insufficient number of samples and return the result with some confidence interval around the probability of each node and edge.

{
\section{{Implications of Approximations in DAG Representation}}
\label{sec:approximation-implications}

The DAG representation serves several important purposes. 
Firstly, it can act as a potentially valuable visualization of consumer preferences. 
Secondly, it has a practical application: using the DAG representation, one can readily identify all top-$j$ sets of frequent types for a given $j$, represented as nodes at level $j$. 
These sets can then be used to  create an assortment containing at least one top-$j$ item for each frequent type{, or} similar assortment design tasks.

 Our findings in Section~\ref{sec:dag-choice} reveal an equivalence between the DAG representation and choice probabilities. Thus, an exact DAG {facilitates} the precise calculation of choice probabilities, and vice versa. 
This in turn enables the application of the DAG framework in optimization problems like assortment planning under various constraints including \emph{user and item fairness}. 
Such constraints ensure that in expectation, the offered set maintains a sufficiently high market share (probability of purchase) for each frequent user type and for each item. See \cite{chen2022fair, chen2023interpolating, lu2023simple} for  recent works that study such constraints.

While an exact DAG representation enables the computation of exact choice probabilities, \AlgDAG computes an approximate DAG, raising the question of how this approximation impacts the computation of choice probabilities.
Observe that there are three sources of approximation as characterized in Theorem~\ref{thm:main}:

\begin{enumerate}
    \item Only prefixes corresponding to frequent types will occur in the representation output by \AlgDAG.
    \item The vertex and edge probabilities are estimated only up to some errors.
    \item The DAG is truncated at level $\numitems = \alpha \numel$.
\end{enumerate}

{The first error source} is most benign.  Rare types  contribute a very small probability (at most $\raresum$ total), so not including {their prefixes} in calculations can change the choice probability estimates by at most $\raresum$.

The impact of the second source can also be bounded. In Proposition~\ref{prop:exact_choice_prob}, each frequent type $\pi$ can contribute at most one term to the sum, because $\pi$ uniquely defines the set $A$ such that $\pi \in \edgetype{A}{z}$, as the set of all items preceding $z$ in the ranking $\pi$.
As a result, the sum can have at most $K$ non-zero terms, and each is estimated, by Theorem~\ref{thm:main}, to within an additive error at most $\epsilon/\numitems + \raresum$. Thus, the total estimation error due to errors in estimates of $\edgeprob{A}{z}$ can be at most $K \cdot (\epsilon/\numitems + \raresum)$. The first term can be made small by decreasing $\epsilon$; for the second term, one needs an assumption that $K \raresum$ is not too large, i.e., the rare types make up only a small fraction of the population in total.

Most problematic  is the truncation at level $\numitems = \alpha \numel$. Indeed, if there is a (very) frequent type $\pi$, and the assortment offered comprises the two least favorite items of $\pi$, the available information will not allow for any accurate estimation of the probability with which this type --- and by extension the population --- chooses one item vs.~the other.
More generally, any assortment $S$ such that all items in $S$ are among the bottom $1-\alpha$ fraction of types whose total probability mass is large enough will not allow an estimation of the probability with which any particular item $z \in S$ is chosen.
However, the converse is also true: if $S$ \emph{does} contain, for each frequent type $\pi$,  at least one item ranked in the top $\alpha$ fraction of items, then the truncation is not problematic from the perspective of computing $q_z(S)$. 
To see this, recall that by Proposition~\ref{prop:exact_choice_prob}, we have $q_z(S)  = \sum_{A: A \cap S = \emptyset} p(\edgetype{A}{z})$. 
Now consider any set $A$ of size more than $\numitems$ with $A \cap S = \emptyset$, whose probability thus was not estimated. Any type $\pi \in \edgetype{A}{z}$ must have all of $A$ preceding all of $\universe \setminus A$, and in particular all of $S$. Because $|A| > \numitems$, this means that no items of $S$ appears in the first $\numitems$ items of $\pi$, so by the assumption, $\pi$ cannot be a frequent type. 
Therefore, $\edgetype{A}{z}$ cannot contain any frequent types for any $A$ that was not estimated, and the estimate of $q_z(S)$ will be accurate up to rare types.

From a practical perspective, assortments under which a frequent type dislikes all items (in the sense of ranking all of them in the bottom $1-\alpha$ fraction of the rankings) are perhaps of less interest to evaluate. 
In particular, this applies if consumers would prefer not to purchase any items in such cases. 
In fact, if we assume that a no-purchase option is always included as an ``item'', and consumers dislike their bottom $1-\alpha$ fraction of items enough to always prefer the no-purchase option, then the no-purchase option will be an element of every set $S$ which is ranked in the top $\numitems$ items of each type $\pi$. 
As a result, the discussion from the previous paragraph exactly applies.

Under the random model, choice sets are also unlikely to face this issue.
For a choice set of size $S$, the probability that a random type ranks all of $S$ in the bottom $1-\alpha$ fraction of items is at most  $(1-\alpha)^{|S|}$. By a union bound over all frequent types, the probability that at least one frequent type is affected is at most $K (1-\alpha)^{|S|}$.
Whenever the choice set $S$ is sufficiently large, this quantity will be small enough that most choice probabilities will be estimated accurately.

Thus, under various natural modeling assumptions, the impact of all types of errors is limited. Nonetheless, we believe that eliminating the third source of error --- i.e., trying to not truncate the DAG at all --- is perhaps the most interesting technical direction for future work.

}

\section{Experiments}
\label{sec:experiment}
{
In this section, we describe two experiments.
The first experiment complements the theoretical discussion in Section \ref{sec:approximation-implications} and focuses the impact of approximation in the DAG representation on the estimated choice probabilities.
The second experiment focuses on the accuracy of \AlgDAG in estimating the true DAG. 
Our goal in the second experiment is to investigate the benefit of actively learning the choice model, compared to learning the best choice model given the available transaction data, as previously done in \citet{farias2013nonparametric}, \citet{van2015market}, and \citet{van2017expectation}.

\subsection{Impact of Approximation  on Estimated Choice Probabilities} \label{sec:Experiment:choice Prob}

\textbf{Setup.} We generate choice models as follows. We consider $5$ frequent types and $n\in \{8, 16\}$ items. The total probability $\raresum$ of noise is  chosen from $\SET{0.001, 0.01, 0.05}$. 
For any  $\raresum $, there are $20$ noisy types.\footnote{Separate experiments have shown that changing the number of rare types does not impact our results given a fixed $\rho$.} 
For each of the frequent and noisy types, a ranking over the $n$ items is chosen i.i.d.~uniformly at random. 
The probability of the frequent types is drawn from a symmetric Dirichlet distribution whose coefficient of variation is $0.1$, 
multiplied by $(1-\raresum)$. 
The probability of the noisy types is distributed as a symmetric Dirichlet distribution whose coefficient of variation is $0.1$, multiplied by $\raresum$. 
For each of $\raresum\in\SET{0.001, 0.01, 0.05}$, $100$ instances of the described choice models are generated.

\textbf{Metric.} For each choice model, we uniformly randomly generate 100 sets of size $k=4$. (The results are similar when k is set to 3 or 5.) We then calculate the average L1 error in the estimated choice probabilities {for} the randomly chosen set. 
More specifically, let $S$ be the (randomly chosen) set and $(\hat{q}_z(S))_{z\in S}$ be the estimated choice probabilities {for the elements in} this set. The L1 error  is $\sum_{z\in S}|\hat{q}_z(S) - q_z(S)|$. For each choice model, we compute the average L1 errors, taken over the {100 draws of} random sets.

\subsubsection{Results} 

\begin{table}[htbp]
    \centering
    \begin{minipage}[b]{0.42\textwidth}
    \centering
    \caption{L1 error of choice probabilities for $\numel=8$ and random sets of size $4$ with $\kappa=0.01$,	and $\delta=0.05$.} \label{table:n8}
 \footnotesize
    \begin{tabular}{cccccccc}
        \toprule 
     \multirow{2}{*}{\textbf{$n_0$}} & \multirow{2}{*}{$\rho$} & \multirow{2}{*}{$\epsilon$} & \multicolumn{5}{c}{\textbf{Statistics}} \\
        \cmidrule{4-8}
        &&& \textbf{Mean} & \textbf{Std} & \textbf{25\%} & \textbf{50\%} & \textbf{75\%} %& \textbf{90\%} 
        \\
        \midrule
     \multirow{6}{*}{3} &  \multirow{2}{*}{ 0.001}& 0.001 & 0.071 & 0.011 & 0.064 & 0.072 &0.079 %& 0.085 
     \\      
        &  & 0.01 & 0.072 & 0.010 & 0.065 & 0.072 & 0.080 %& 0.085 
        \\
        & \multirow{2}{*}{   0.01} & 0.001 & 0.077 & 0.012 & 0.070 & 0.078 & 0.084 %& 0.091 
        \\
        &  & 0.01 & 0.076 & 0.010 & 0.070 & 0.077 & 0.082 %& 0.087 
        \\
        & \multirow{2}{*}{   0.05} & 0.001 & 0.087 & 0.012 & 0.080 & 0.086 & 0.095 %& 0.0101 
        \\
        &  & 0.01 & 0.085 & 0.012 & 0.077 & 0.085 & 0.092 %0.099
        \\
        \cmidrule(lr){1-8}
        %\cmidrule{2-9}
        %& \multirow{2}{*}{4} & 0.001 & 0.001 &  0.0005 & 6e-5 & 0.0005 & 0.0005 & 0.0006 & 0.0006 \\
       % && 0.001 & 0.01 & 0.0015 & 0.0001 & 0.0014 & 0.0015 &  0.0016 & 0.0016 \\
     \multirow{6}{*}{4} & \multirow{2}{*}{0.001} & 0.001 & 0.014 & 0.005 & 0.010 & 0.014 &  0.017 %& 0.021 
     \\
    &  & 0.01 & 0.014 & 0.005 & 0.010 & 0.014 & 0.018 %& 0.021
    \\ 
 &   \multirow{2}{*}{0.01} & 0.001 & 0.020 & 0.006 & 0.015 & 0.020 & 0.023 %& 0.028
 \\ 
  &    & 0.01 & 0.020 & 0.005 & 0.016 & 0.019 & 0.022 %& 0.026
  \\ 
  &   \multirow{2}{*}{0.05} & 0.001 & 0.031 & 0.006 & 0.027 & 0.031 & 0.035 %& 0.040
 \\ 
  &    & 0.01 & 0.032 & 0.005 & 0.028 & 0.031 & 0.036 %& 0.039
  \\ 
     \cmidrule(lr){1-8}
      %  \cmidrule{2-9}   
         \multirow{6}{*}{5} & \multirow{2}{*}{0.001} & 0.001 & 0.001 & 0.000 & 0.001 & 0.001 &  0.001 %& 0.001 
     \\
    &  & 0.01 & 0.001 & 0.000 & 0.001 & 0.001 & 0.001 %& 0.001
    \\ 
 &   \multirow{2}{*}{0.01} & 0.001 & 0.006 & 0.001 & 0.005 & 0.005 & 0.006 %& 0.006
 \\ 
  &    & 0.01 & 0.006 & 0.001 & 0.005 & 0.006 & 0.006 %& 0.006
  \\ 
  &   \multirow{2}{*}{0.05} & 0.001 & 0.018 & 0.003 & 0.016 & 0.017 & 0.019 %& 0.020
 \\ 
  &    & 0.01 & 0.017 & 0.002 & 0.016 & 0.017 & 0.019 %& 0.020
  \\ 
        \bottomrule
    \end{tabular}
    \end{minipage}
    \hfill
    \begin{minipage}[b]{0.42\textwidth} 
        \centering \footnotesize
           \caption{L1 error of choice model for $\numel=16$ and random sets of size $4$ with $\kappa=0.01$,
	and $\delta=0.05$.} \label{table:n16}
  \begin{tabular}{ccccccccc}    
        \toprule
        \multirow{2}{*}{\textbf{$n_0$}} & \multirow{2}{*}{$\rho$} & \multirow{2}{*}{$\epsilon$} & \multicolumn{5}{c}{\textbf{Statistics}} \\
        \cmidrule(lr){4-8}
        &&& \textbf{Mean} & \textbf{Std} & \textbf{25\%} & \textbf{50\%} & \textbf{75\%} \\
        \midrule
        \multirow{6}{*}{6} & \multirow{2}{*}{0.001} & 0.001 & 0.118 & 0.014 & 0.108 & 0.117 & 0.126 \\
        &  & 0.01 & 0.118 & 0.015 & 0.106 & 0.118 & 0.127 \\
        & \multirow{2}{*}{0.01} & 0.001 & 0.121 & 0.014 & 0.110 & 0.121 & 0.130 \\
        &  & 0.01 & 0.122 & 0.015 & 0.111 & 0.124 & 0.134 \\
        & \multirow{2}{*}{0.05} & 0.001 & 0.142 & 0.012 & 0.133 & 0.141 & 0.150 \\
        &  & 0.01 & 0.141 & 0.013 & 0.132 & 0.141 & 0.150 \\
        \cmidrule(lr){1-8}
           \multirow{6}{*}{8} & \multirow{2}{*}{0.001} & 0.001 & 0.040 & 0.009 & 0.034 & 0.039 & 0.045 \\
        &  & 0.01 & 0.039 & 0.008 & 0.033 & 0.039 & 0.045 \\
        & \multirow{2}{*}{0.01} & 0.001 & 0.046 & 0.008 & 0.040 & 0.045 & 0.050 \\
        &  & 0.01 & 0.048 & 0.008 & 0.042 & 0.048 & 0.051 \\
        & \multirow{2}{*}{0.05} & 0.001 & 0.066 & 0.010 & 0.059 & 0.065 & 0.072 \\
        & & 0.01 & 0.066 & 0.008 & 0.060 & 0.067 & 0.072 \\
        \cmidrule(lr){1-8}
           \multirow{6}{*}{10} &\multirow{2}{*}{0.001} & 0.001 & 0.009 & 0.004 & 0.007 & 0.009 & 0.011 \\
        &  & 0.01 & 0.009 & 0.004 & 0.007 & 0.009 & 0.012 \\
        & \multirow{2}{*}{0.01} & 0.001 & 0.016 & 0.004 & 0.014 & 0.016 & 0.018 \\
        &  & 0.01 & 0.016 & 0.005 & 0.013 & 0.016 & 0.019 \\
        & \multirow{2}{*}{0.05} & 0.001 & 0.038 & 0.005 & 0.035 & 0.038 & 0.041 \\
        &  & 0.01 & 0.037 & 0.005 & 0.034 & 0.037 & 0.041 \\
        \bottomrule
    \end{tabular}
    \end{minipage}
\end{table}

Tables~\ref{table:n8} and \ref{table:n16} present the average (mean) of the L1 errors, along with their standard deviation (std), and 25th, 50th, and 75th percentiles for $\numel=8$ and $\numel=16$, respectively. This analysis is conducted for values of $\raresum \in \{0.001, 0.01, 0.05\}$, $\epsilon\in \{0.001,  0.01\}$, $\kappa=0.01$, and $\delta =0.05$. 
In Table~\ref{table:n8}, for certain values of $\numitems$, the results for $\raresum=0.001$ are omitted, because all numbers are zero after rounding them to three decimal places. 
In Table~\ref{table:n8}, corresponding to $\numel=8$, we consider $\numitems$ values from the set $\{3, 4, 5\}$. 
These values are then doubled in Table~\ref{table:n16} when the number of items is $\numel=16$.
We notice relatively small L1 errors of choice probabilities. For example, even when $\numel=8$ and $\numitems=3$ with a noise level of 5\%,  the mean L1 error for a random set of size $4$ is only $0.031$, with 75\% of the samples having an L1 error of less than $0.035$.
Further observations include:

\begin{itemize}[leftmargin=*]
\item \textbf{Impact of $\rho$ and $\epsilon$:} As expected, the mean and standard deviation increase in $\raresum$ and $\epsilon$, as there is more noise in the model or fewer samples are drawn. However, this increase is only small.

\item \textbf{Impact of truncation:} As the truncation level $\numitems$ increases (so the DAG representation has more levels), there is a general trend of decreasing mean L1 error, standard deviation, and percentiles. Consistent with the discussion in Section~\ref{sec:approximation-implications}, the impact of $\numitems$ is significantly larger than that of $\rho$ and $\epsilon$.

\item \textbf{Impact of $\numel$:} Comparing Tables \ref{table:n8} and \ref{table:n16}, doubling the number of items from $\numel=8$ to $\numel=16$ leads to higher mean errors and increased variability in the errors. {In the limit of large $\numel$, the ratio $\numitems/\numel$ is the primary indicator of errors, because it captures the probability that a random element of the set is in the DAG.
However, for small $\numel$, because the set $S$ contains distinct elements, the probability that none of them is in the DAG is significantly smaller than the estimate $(1-\numitems/\numel)^k$, implying a higher market share (see next paragraph) and thus lower error.}
\end{itemize}

Tables~\ref{table:n8} and \ref{table:n16} show that truncation is the primary source of L1 errors.
To shed light on the impact of truncation, in Figure~\ref{fig:MS}, we present the L1 error of choice probabilities versus the \emph{market share} of a set of size $4$ when $\numel=16$, $\numitems=8$, and $\raresum \in \{0.01, 0.05\}$. 
Here, the market share of a set $S$ given the truncation level $\numitems$ is equal to the probability $\sum_{\pi: \{\pi(1), \ldots, \pi(\numitems)\} \cap S \neq \emptyset} p(\pi)$ that at least one of the items in $S$ is in the first $\numitems$ layers of the DAG. 
The market share of $S$ in this sense can be considered as a measure of the extent to which the assortment provides at least one attractive item (within the top $\numitems$ levels) to customers.

Figure~\ref{fig:MS} shows that the larger the market share, the smaller the L1 error. In fact, the relationship between the L1 error and the market share of a set is linear.
The same observation is also confirmed for every combination of $\numel$ and $\numitems$ presented
in Tables~\ref{table:n8} and \ref{table:n16}. 
The impact of $\rho$, which can be viewed as the observed deviation from the linear model, is quite negligible; this is consistent with our findings in Tables~\ref{table:n8} and \ref{table:n16}. 
Similar observation hold for the dependence on $\epsilon$ (not shown here).

\begin{figure}[htb]
     \centering         
     \includegraphics[width=0.35\textwidth]{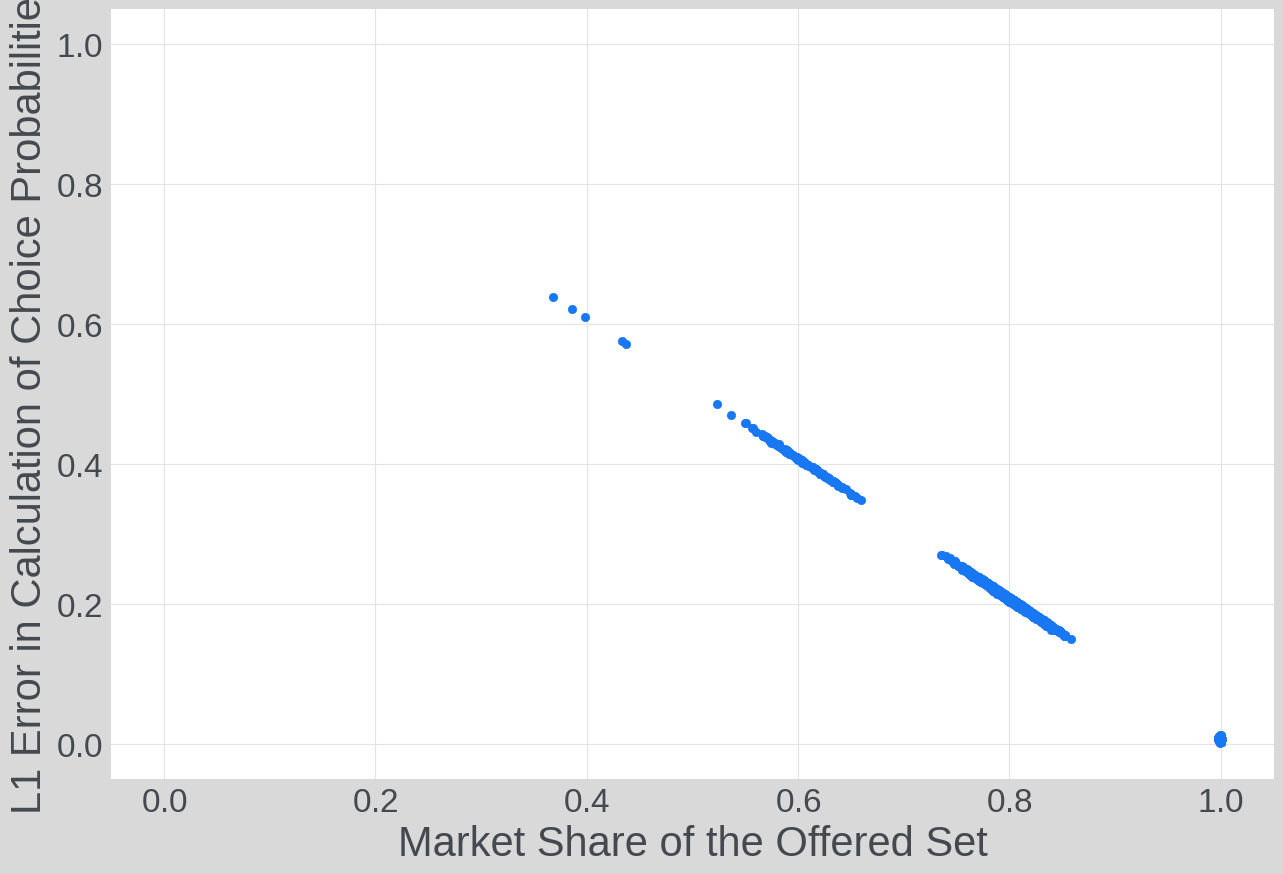}
    \includegraphics[width=0.35\textwidth]{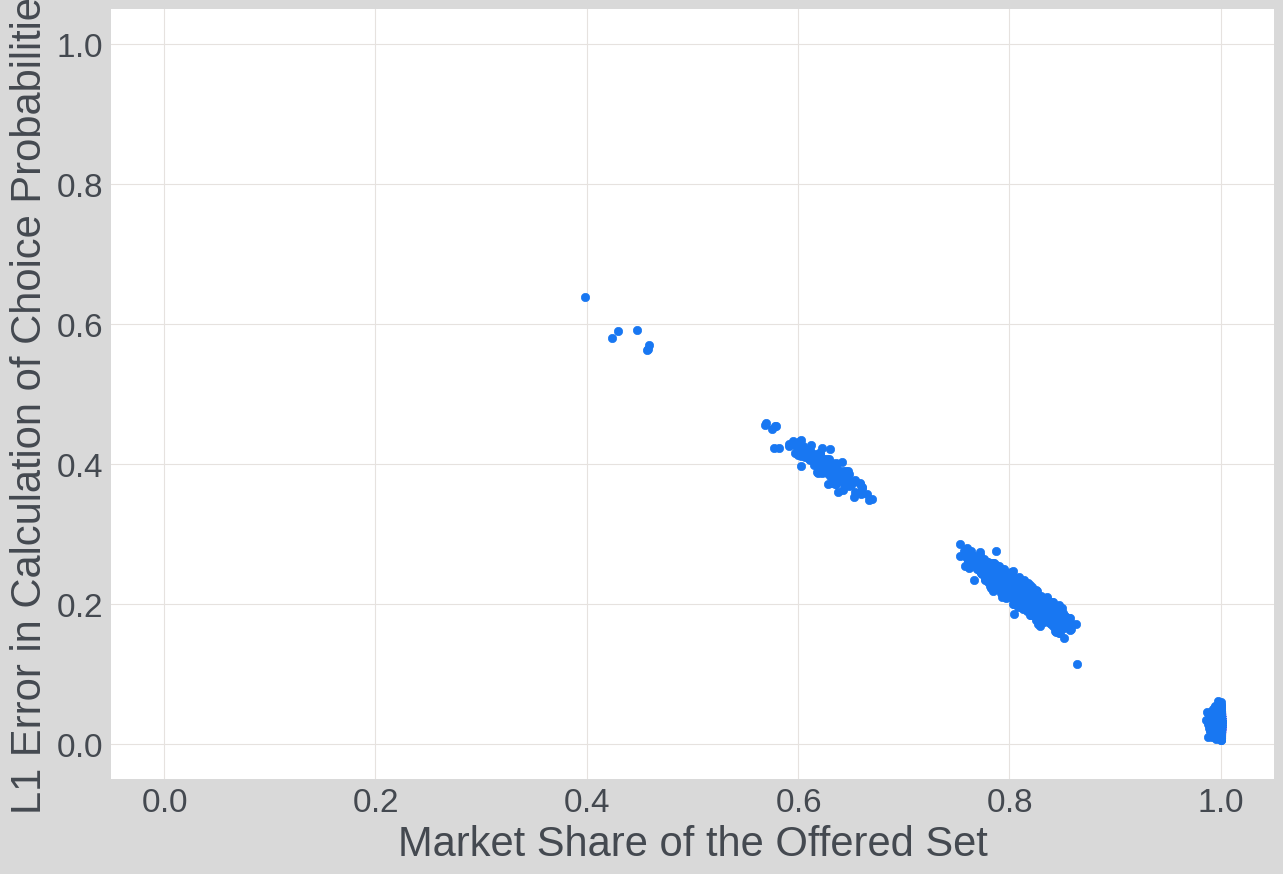}
    \caption{L1 error in choice probabilities vs.~market share of the set. In both figures, $\numel=16$, $\numitems=8$, $\epsilon=0.01$, and $\kappa=0.01$. 
    The only difference is that in the left figure, $\raresum=0.01$, while in the right figure, $\raresum=0.05$.} \label{fig:MS}
\end{figure}
}

\subsection{The Value of Active Learning}

In our second set of experiments, we evaluate the performance of \AlgDAG empirically on both synthetic and real-world datasets, shedding light on the value of active learning. 
Our experiments show that, in addition to being easy to implement, \AlgDAG can recover the true DAG accurately (with low estimation errors) with high probability using a reasonable number of samples. They further show that actively learning the choice model gives better performance empirically than learning the model directly from predetermined/given transaction data. 

\textbf{Synthetic setup.} For the synthetic analyses, we use a setup similar to the first experiment, with a few minor differences. Due to computational challenges posed by the benchmark algorithm (defined shortly), we only conduct our analysis for $\numel=8$. 
(In Appendix~\ref{sec:large-exp}, we further evaluate \AlgDAG for a larger number of types and items, to observe how \AlgDAG performs when the instance size scales up.) 
Similar to the first experiment, the total probability $\raresum$ of noise is chosen from $\{0, 0.001, 0.01, 0.05\}${; note that we add the noiseless condition $\raresum=0$ here}. 

\textbf{Real-world setup.} 
We further evaluate the performance of \AlgDAG on the sushi data set (\cite{kamishima2005supervised}). This publicly available data set describes preferences over $10$ popular types of sushi based on a consumer survey that involved $5000$ people specifying their exact ranking. In the sushi data set, unlike the data one would acquire in an active learning manner, a user declares his/her entire ranking, whereas from transaction data, we only get to see the most popular sushi among those offered. 

In the sushi data set, there are $4926$ distinct rankings. Taking $\kappa = 0.0001$, there are $90, 561, 1800$, and $3273$ frequent prefixes for the top $\numitems= 2,3,4$ and $5$ items, respectively, all with zero noise probability, i.e. $\raresum = 0$.\footnote{Increasing the threshold $\kappa$ for frequent types causes the total probability of noise to be much larger than $\kappa$; hence we choose $\kappa = 0.0001$ to be our threshold. For instance, when $\kappa = 0.0002$, we have $\raresum = 0.4754$, which is more than 2000 times the threshold for frequent types.} 
Our choice of $n_0$ is mainly due to our benchmark algorithm (we elaborate on this more in the next section), which is computationally inefficient when estimating more than the top $5$ items in our experiments. 
Observe that in this setup, there are a lot more distinct types and prefixes compared to the synthetic setting. This causes a lot more interference, which makes it harder for \AlgDAG to estimate each node properly. Furthermore, the set of rankings  is obviously not generated i.i.d.~uniformly at random. We  use the empirical frequency of each prefix in the data set as our ground truth, so for each $\numitems\in \SET{2, 3, {4}, 5}$ that we consider, we get the set of types of size $\numitems$ with its corresponding probability distribution.

\textbf{Metrics.} We compare the true DAG representation of the choice model with the DAG produced by \AlgDAG for various values of $\numitems$ (the number of top positions to estimate). We consider two metrics to measure the performance of \AlgDAG. The first one is the number of nodes in the DAG that are added wrongly (false positive) or not inferred (false negative). The second one is the maximum probability discrepancy/difference between the true DAG and the inferred DAG.

\subsubsection{Benchmark}
As our benchmark, we implemented the primal linear program (LP) from \cite{farias2013nonparametric} to find the set of types. \cite{farias2013nonparametric} summarize the given transaction data in an $m$-dimensional vector $\bm{y} = A\bm{\lambda}$, where $A\in\SET{0,1}^{m\times n!}$ represents the relationship between the observed data ($\bm{y}$) and the choice model.\footnote{Specifically, the \Kth{i} element of $\bm{y}$, which is associated with a set $S$ and an item $z$, represents the fraction of consumers who choose an item $z$ given assortment $S$. Then, the element in the \Kth{i} row and \Kth{j} column of matrix $A$, where the \Kth{i} row is associated with a set $S$ and an item $z$, is $1$ if consumers with type $j$ prefer $z$ among all items in $S$, and $0$ otherwise.} Here, the choice model is represented by the vector $\bm{\lambda}$, which is the variable representing the probability distribution over all possible types of preferences. Unlike \AlgDAG, which builds the DAG level by level using active learning, this method determines the distribution over the whole ranking (containing $n$ items) using an offline data set. That is, in our synthetic setting, for a given offline data set, it returns a distribution over all possible $8! = 40320$ candidate rankings. In the sushi setting, since we only run the benchmark algorithm to estimate the top $5$ items, we limit our type space to  the set of all possible rankings of five items; in this case, the benchmark returns a distribution over all $5!\cdot\binom{11}{5} = 55440$ possible top $5$ items. Thus, in the sushi data set, we use a smaller search space for the types instead of searching over all $11!=39,916,800$ rankings since \AlgDAG also only estimates up to the top $5$ items. This is beneficial for the calculation complexity for the benchmark since it reduces the number of decision variables in the LP.\footnote{In the synthetic setting, the problem is much easier because the number of types is at most $35$, while in the sushi setting, there are at least $3273$ distinct types when considering the top $5$ items. Thus, in the synthetic data set, unlike the sushi data set, we consider all $n!$ rankings in the benchmark LP.} 
{See more details about the benchmark implementation in Section~\ref{sec:benchmark}.} 

%For both  data sets, we do the following for the benchmark: (1) solve the primal LP fully (without constraint sampling) using the \texttt{Gurobi} solver, instead of solving the dual approximately (as done in \cite{farias2013nonparametric}), and (2) impose sparsity on $\bm{\lambda}$, the probability distribution of types. These give the primal LP algorithm an extra benefit{:} solving the primal fully permits searching over a bigger space, and the sparsity constraint helps the LP deal with the indistinguishability problem. Furthermore, we relax the LP constraint to overcome infeasibility, by incorporating $1/\sqrt{\numsamples}$ standard error for the empirically calculated probability that each item in the assortment is chosen, where $\numsamples$ is the number of samples used per offered assortment in \AlgDAG and the benchmark. So instead of having the constraint $\bm{y} = A\bm{\lambda}$, we have $\hat{\bm{y}}-1/\sqrt{{\numsamples}} \cdot \bm{1} \leq A\bm{\lambda}\leq \hat{\bm{y}}+1/\sqrt{{\numsamples}} \cdot \bm{1}$. We do so because the primal LP often does not yield any solutions {under} the constraint $\hat{\bm{y}} = A\bm{\lambda}$ where $\hat{\bm{y}}$ is empirically calculated using the same number of samples as \AlgDAG.\footnote{This is mainly because the system of equations can be overspecified, so if we use an estimated version of $\ybf$ that might have some error, two or more systems of equations can result in different values, leading to a contradiction.}

{
To evaluate the benchmark, we generated two types of offline transaction datasets for the synthetic setting. The first dataset aims to replicate common stockout scenarios observed in the real world, considering the impact of stockouts on demand and revenue \citep{caine1976optimal, fitzsimons2000consumer, kim2011consumer}. Here, we assume that $3$ out of $\numel$ items are unavailable, leaving consumers with $(\numel-3)$ items to choose from. This results in $\binom{\numel}{3}$ distinct assortments, with each assortment shown to $\numsamples$ consumers. 
With $8$ items, this yields $56$ distinct assortments. 
The second dataset comprises pairwise comparison data representing the fraction of consumers who prefer one item over another, as seen in previous research \citep{farias2013nonparametric}. This dataset contains $\binom{\numel}{2}$ distinct assortments, resulting in $28$ distinct assortments when there are $8$ items.\footnote{
{When applying \AlgDAG to the synthetic dataset, we find that it uses a maximum of 20 distinct assortments, even with $\numitems=5$. This suggests that \AlgDAG employs fewer distinct queries compared to offline datasets. However, through active learning, as we will show later, \AlgDAG significantly outperforms the benchmark.}}
}

{
In the sushi setting, we generated two types of offline transaction data: assortments with $3$ stockouts (resulting in $8$ items) and pairwise assortments. However, when using pairwise assortments, the linear programs often remained unsolved due to an undetermined linear system, especially with the large number of distinct rankings in the sushi dataset ($4926$). Therefore, we focus solely on the results from the offline assortments with $3$ stockouts for the sushi dataset to provide additional advantage to the offline benchmark. Nevertheless, as we will demonstrate later, even in the synthetic dataset, the benchmark performed worse on the pairwise dataset compared to the $3$-stockout dataset.}

\subsubsection{\AlgDAG Implementation}
In the synthetic setting, we want to estimate the top $\numitems = 3,5$ positions when the number of items is $8$, while in the sushi setting, we want to estimate the top $\numitems = 2,3,4,5$ positions. In the synthetic data set, we consider as frequent the types whose probability is at least $\kappa = 0.01$; this is a consequence of our choice of Dirichlet parameters for the probability distribution of the frequent types in the synthetic setting.\footnote{Recall that in the synthetic setting, the probability of the frequent types is drawn from a symmetric Dirichlet distribution whose coefficient of variation is $0.1$, multiplied by $(1-\raresum)$. The implication is that most of the probability of a frequent type generated is bigger than $0.01$, so we can set $\kappa = 0.01$.} Meanwhile, in the sushi setting, we choose $\kappa = 0.0001$ by looking at the smallest type probability in the ranking. Note that in our experiment, we do not always have $\raresum < \kappa/4$ for our chosen $\raresum \in \SET{0,0.001,0.01,0.05}$. While this assumption is required for our theoretical guarantee in Theorem~\ref{thm:main}, by violating it, we can observe how robust \AlgDAG is. %\ngcomment{Some texts are omitted here.} %Our accuracy parameters {are} $\epsilon = 0.001$ 
%\ngcomment{revisit this part; we are setting $\epsilon$ to be large given that $m= 10,000$.} 
%for the synthetic setting and $\epsilon = 1e-5$ in the sushi setting.  $\epsilon$ is very small for the sushi data because as we explain above, we choose the threshold $\kappa$ to be $0.0001$, and we want $\epsilon\ll\kappa$. We further set $\delta =0.05$, meaning that we want each node's estimate to be within $\epsilon$ of its true probability, with probability at least $95\%$. 

In simulating active learning, given a set of (true) synthetic parameters, one consumer is drawn in each round, and the item she chooses based on her preference is observed. The algorithm gets to decide which assortment is offered for each consumer. %Although the number of samples to use is specified in \AlgIE, for practicality, we limit the number of samples for each \AlgIE run to be at most $10k$, i.e., $\numsamples \leq 10k$. 
To run \AlgIE, we set the number of samples for each \AlgIE run to be at most $10k$.
Such a choice would likely correspond to how a practitioner would use \AlgIE, and lets us observe the impact of a number of samples that is significantly smaller than that required for theoretical guarantees.
%\ngdelete{This allows us to see how \AlgDAG and \AlgIE perform with a smaller number of samples than required for the theoretical guarantees.}  %\ngcomment{ revisit the previous sentence. Once we choose a value for $m$, we don't need to set a value for $\epsilon$. But, any value for $m$ can be mapped to accuracy parameters $\epsilon$ and $\delta$.} 
The experiment is run $100$ times for each value of $\numitems$.
\begin{table}[htb]
\footnotesize
\centering
\begin{tabular}{r|c|c|c|c|c|c}
\multirow{2}{*}{Noise} & \multicolumn{2}{c|}{\AlgDAG, $\numsamples \leq 10k$} & \multicolumn{2}{c|}{\cite{farias2013nonparametric}, $3$ stockout} & \multicolumn{2}{c}{\cite{farias2013nonparametric}, pairwise data} \\
\cline{2-7}
& $\numitems = 3$ & $\numitems = 5$ & $\numitems = 3$ & $\numitems = 5$ & $\numitems = 3$ & $\numitems = 5$ \\
\hline
 0 &$0.0083\pm0.00040$ &$0.0092\pm0.00041$ &$0.2866\pm 0.01100$ &$0.2988\pm 0.01214$ &$0.3785\pm 0.00916$ &$0.3884\pm 0.01032$\\
0.001 &$0.0099\pm0.00033$ &$0.0111\pm0.00032$ &$0.3054\pm0.01110$ &$0.3359\pm0.01183$ &$0.3953\pm0.01132$ &$0.4026\pm0.01201$\\
 0.01 &$0.0103\pm0.00036$ &$0.0113\pm0.00038$ &$0.3100\pm0.01134$ &$0.3585\pm0.01423$ &$0.4280\pm0.01141$ &$0.4186\pm0.01160$\\
0.05 &$0.0105\pm0.00039$ &$0.0114\pm0.00039$ &$0.3320\pm0.01321$ &$0.3878\pm0.01477$ &$0.4058\pm0.01078$ &$0.4414\pm0.01178$\\
 \end{tabular}
 \vspace{0.25cm}
 \caption{Maximum probability discrepancy, over all nodes, between the true DAG and the DAG returned by \AlgDAG or the benchmark, when there are $5$ frequent types and $8$ items, averaged over $100$ instances. Here, $\numsamples = 10k$.}
 \label{tab:avgmaxdisc_our}
\end{table}

\subsection{Results}

\subsubsection{Synthetic Data Set.} We present our results in Tables~\ref{tab:avgmaxdisc_our} and \ref{tab:avgperdiff_our}.  In Table~\ref{tab:avgmaxdisc_our}, we look at the maximum probability discrepancy out of all nodes between the DAG returned by the algorithm and the true DAG.  \AlgDAG outperforms the benchmark for both offline data sources: the 3 stockout and pairwise data. In general, the results for both algorithms (\AlgDAG and the benchmark) get worse when the noise probability increases. Further, {in} Table~\ref{tab:avgmaxdisc_our}, we see that the maximum probability discrepancy for the DAG returned by \AlgDAG, which is attained when the noise probability is the largest, i.e., $0.05$, is at most $0.00892$ on average (when $\numitems=5$). Moreover, the maximum noise discrepancy is larger for $\numitems=5$ than for $\numitems=3$. This shows that as the number of top items being inferred increases, the problem becomes harder. This is similar to the behavior predicted by Theorem~\ref{thm:main}, where more samples are needed to obtain the same level of accuracy. Comparing the benchmarks utilizing two distinct data sources --- specifically, the 3 stockout data set and the pairwise data set --- we observe that the benchmarks exhibit superior performance when employing the 3 stockout offline data set as opposed to the pairwise data set. The former has more variation in the transaction data with $56$ distinct assortments, therefore containing more information.

\begin{table}[htb]
\footnotesize
\centering
\begin{tabular}{r|c|c|c|c|c|c}
\multirow{2}{*}{Noise} & \multicolumn{2}{c|}{\AlgDAG, $\numsamples \leq 10k$} & \multicolumn{2}{c|}{\cite{farias2013nonparametric}, $3$ stockout} & \multicolumn{2}{c}{\cite{farias2013nonparametric}, pairwise data} \\
 \cline{2-7}
& $\numitems = 3$ & $\numitems = 5$ & $\numitems = 3$ & $\numitems = 5$ & $\numitems = 3$ & $\numitems = 5$ \\
\hline
0 &$1.58\%\pm0.149\%$ &$1.73\%\pm0.209\%$ &$30.73\%\pm0.315\%$ &$35.30\%\pm0.633\%$ &$36.71\%\pm0.964\%$ &$41.90\%\pm0.796\%$\\
0.001 &$4.67\%\pm0.154\%$ &$5.04\%\pm0.228\%$ &$31.51\%\pm0.333\%$ &$41.01\%\pm0.744\%$ &$37.19\%\pm0.991\%$ &$45.25\%\pm0.972\%$\\
0.01 &$5.40\%\pm0.167\%$ &$5.68\%\pm0.234\%$ &$37.67\%\pm0.366\%$ &$41.27\%\pm0.656\%$ &$44.61\%\pm1.278\%$ &$46.48\%\pm0.994\%$\\
0.05 &$5.73\%\pm0.207\%$ &$6.04\%\pm0.252\%$ &$38.12\%\pm0.363\%$ &$49.33\%\pm0.770\%$ &$46.99\%\pm1.417\%$ &$52.51\%\pm1.103\%$
\end{tabular}\vspace{0.25cm}
\caption{Percentage of vertices that are different between the true DAG and the DAG returned by \AlgDAG or the benchmark, averaged over $100$ instances. Here, $\numsamples = 10k$.}
 \label{tab:avgperdiff_our}
\end{table}

In Table~\ref{tab:avgperdiff_our}, we look at the percentage of vertices that are added wrongly or not inferred, i.e., the number of vertices that appear only once (either in the true DAG or in the DAG returned by \AlgDAG, but not in both) on level $\numitems$, as a percentage of the number of vertices on level $\numitems$ in the true DAG formed by the frequent types. Compared with the benchmark, \AlgDAG produces fewer vertices that are being added wrongly or not inferred compared to the benchmark, with both data sources.\footnote{Additional experiments --- not discussed here --- show that the number of different vertices goes down to $0$ when we increase $\numsamples$ to $50k$.} For example, when $\numitems=5$ and the noise level is $0.001$, while the percentage of different vertices is $5.04\%$ in \AlgDAG, it is $41.01\%$ for the benchmark with 3 stockout data and $45.25\%$ for the benchmark with pairwise data. Further, as before, the benchmark with 3 stockout data outperforms the one with pairwise data due to the existence of more variation in the data. As expected, for both \AlgDAG and the benchmark, the percentage of different vertices gets bigger as the total noise probability increases.

\subsubsection{Sushi Data Set.}
\label{sec:sushires}
\begin{table}[htb]
 \centering
\footnotesize
\begin{tabular}{l|c|c|c|c|c|c|c|c}
\multirow{2}{*}{} & \multicolumn{4}{c|}{\AlgDAG, $\numsamples \leq 10k$} & \multicolumn{4}{c}{\cite{farias2013nonparametric}, $3$ stockout} \\\cline{2-9}
& $\numitems = 2$ & $\numitems = 3$ & $\numitems = 4$ & $\numitems = 5$ & $\numitems = 2$ & $\numitems = 3$ & $\numitems = 4$ & $\numitems = 5$\\
max.~prob. &{$0.0110\pm$} &{$0.0258\pm$} &{$0.0411\pm$} &{$0.0765\pm$} &{$0.0203\pm$} &{$0.0320\pm$} &{$0.0472\pm$} &{$0.0950\pm$}\\
differences&$0.00023$&$0.00054$&$0.00073$&$0.00134$&$0.00055$&$0.00090$&$0.00105$&$0.00154$\\ \hline
\% of different &{$1.67\%\pm$} &{$8.88\%\pm$} &{$9.45\%\pm$} &{$9.58\%\pm$} &{$23.36\%\pm$} &{$56.77\%\pm$} &{$74.46\%\pm$} &{$86.16\%\pm$}\\
vertices&$0.174\%$&$0.209\%$&$0.124\%$&$0.169\%$&$0.614\%$&$0.535\%$&$0.435\%$&$0.476\%$\\
 \end{tabular}
\vspace{0.5cm}
\caption{Maximum probability discrepancy over all nodes {and percentage of incorrect vertices}, between the true DAG and the DAG returned by \AlgDAG or the benchmark, averaged over $100$ instances run on the sushi data.}
\label{tab:sushi_res_tab}
\end{table}

{Table~\ref{tab:sushi_res_tab} summarizes our results on the sushi data set, showing} the maximum probability differences and the percentages of different vertices between the DAG returned by \AlgDAG and the true DAG, averaged over $100$ runs. \AlgDAG far outperforms the benchmark in both metrics{: it} returns node probabilities within $0.0765$ of the actual probability on average when we look at the fifth level (i.e., $\numitems = 5$), compared to $0.0950$ for the benchmark. Similarly, when we look at the number of vertices that are either added wrongly or not inferred, on average, the percentage for \AlgDAG is nine times smaller than for the benchmark, for $\numitems = 5$. These results are surprising since we are only using $\numsamples = 10k$; they indicate that even with such a small number of samples (far smaller than what is needed in the theoretical guarantee), \AlgDAG can estimate the true distribution of types with good accuracy in the sushi data. Note that this holds even though the types are not drawn i.i.d.~uniformly at random for the sushi data.

\subsubsection{Concluding Remarks for the Second Set of Experiments.}
Overall, we observe that the performance of both algorithms, \AlgDAG and the offline benchmark, gets worse as the total probability of noise (i.e., $\raresum$) gets larger or the number of items being estimated (i.e., $\numitems$) gets bigger. Furthermore, the experiments show that \AlgDAG can get a small probability discrepancy and few ``wrong’’ vertices even with a small number of samples ($\numsamples=10k$). 
When the number of types increases, the performance gets worse but is still comparable, which can be seen in the results for the sushi data set  and the bigger synthetic data set with $15$ frequent types and $20$ items (Appendix~\ref{sec:large-exp}). \AlgDAG also performs well when the number of items is larger than the number of types, as is the case, for example, in the sushi data; this is true as long as there are not too many types with similar preferences causing a lot of interference when \AlgDAG is run. 
In all scenarios, \AlgDAG outperforms the benchmark despite using a small number of distinct queries.

\section{Conclusion, Discussions, and Future Directions}
\label{sec:conclusion}
Non-parametric choice models, {while able}  to capture general choice models, are difficult to estimate using only offline transaction data. This is primarily due to the lack of variation in the data. In many settings, especially on online platforms, one can influence the process of acquiring data based on past observations; this is also known as an active learning process. Motivated by this, we studied the problem of estimating a non-parametric choice model using active learning. In addition to product ranking, non-parametric choice models have also been used in other applications, such as a non-parametric joint assortment and price choice model with consideration set (\cite{jagabathula2017nonparametric}) and non-parametric demand for dynamic pricing (\cite{perakis2019dynamic}). 
In future work, it would be interesting to see how an active learning approach applies to those problems.

We present a negative result showing that such non-parametric choice models might be indistinguishable regardless of the sequence of assortments offered in the active learning process. To overcome this obstacle, we present a directed acyclic graph (DAG) representation of the choice model which captures as much as can be uniquely identified about the distribution when the types of consumers are unobservable and each consumer only makes one choice. {We show that the (exact) DAG representation can be used to compute the (exact)  choice probabilities and vice versa, implying that this representation  encapsulates all the information regarding the choice model that can be deduced from the {available} data.}

We then design an efficient algorithm to  accurately estimate the DAG representation of non-parametric choice models using a polynomial-sized data set. The algorithm learns the distribution over the most popular items and their corresponding probabilities by actively choosing assortments to offer the next arriving consumer based on past observations. The algorithm employs an inclusion-exclusion method to better estimate each ranking probability. We also show that {\AlgDAG} outperforms a non-active learning algorithm on synthetic and real-world (sushi) data sets.

Our work raises several immediate questions for technical follow-up work. First, the factor $\numitems$ in the error bound of Theorem~\ref{thm:main} seems like it could perhaps be avoided. Notice, however, that rare types could indeed impact the estimate of {$\edgeprob{A}{z}$} for multiple $A$, for instance those rare types that rank $z$ first (or early). 
Similarly, while for many practical applications, it may suffice to learn the first $\numitems = \alpha \numel$ items of each ranking (see the discussion in Section~\ref{sec:approximation-implications}), both from a theory and practice perspective, it would be interesting to understand whether learning the entire rankings is possible. To learn about the last few items in rankings, it is necessary to offer small sets to consumers, but if types are sufficiently frequent, it may be possible to obtain data on those items. However, the approach based on Set Cover is unlikely to succeed: when $\numitems$ is large, it is indeed likely that the required set covers are large, and hence, the Inclusion-Exclusion formula requires exponentially many terms, and thus exponentially small error in each term to compensate.

Our polynomially small bounds on the required samples were only guaranteed to apply under the \randmodel{\kappa}{\raresum}. It would be interesting to investigate whether small sample bounds can be obtained under a fully adversarial model. Some preliminary work suggests that this may be difficult. Yet, although there is no polynomial sample bound for the general model, AlgDAG still requires a small number of samples as long as the approximate solution to the Set Cover instance is small.

Experimentally, it would be desirable to evaluate the performance of our algorithms, as well as alternatives, on larger and richer data sets. Ultimately, the goal should be to deploy them in a real-world setting and have them interact with actual consumers. In such a setting, there may be other considerations, such as the cost of exploration (i.e., offering consumers sets of items that are ranked low by many types, to disambiguate the rankings). Many other practical concerns such as inventory constraints are likely to arise.

Fundamentally, we view our work as a first step towards a much more comprehensive evaluation of active learning in the context of choice models and related settings.
Our work opens up new avenues for future work, and we hope that it will lead to studying the framework in other Operations Research problems, such as product ranking and assortment planning. 

\ACKNOWLEDGMENT{N.G. was supported in part by the Young Investigator Program (YIP) Award from the Office of Naval Research (ONR) N00014-21-1-2776 and the MIT Research Support Award.  D.K. was  supported in part by ARO MURI grant W911NF1810208. 
We thank Antoine Desir and Ningyuan Chen for their help and  insightful comments.}

\bibliographystyle{plainnat}
\bibliography{refs}
\newpage
\begin{APPENDIX}{}
\label{sec:appendix}
\section{{Shorter Proofs Omitted in Main Body}}

\subsection{Proof of Theorem \ref{thm:impossible}}\label{sec:proof:thm:impossible}

{\impossibility*}

\begin{proof}
Assume that $\pi$ and $\pi'$ are $i$-indistinguishable for $2 \leq i \leq \numel-2$. We write $q = p(\pi)$ and $q' = p(\pi')$.  Assume that $q\leq q'$. We write $\toppi = \{ \pi(1), \ldots, \pi(i) \}$ for the set of the top-$i$ items common to $\pi$ and $\pi'$.
Let $\overline{\pi}$ be the ranking that is the same as $\pi$ in the first $i$ positions and the same as $\pi'$ in the rest. Similarly, let $\overline{\pi}'$ be the ranking that is the same as $\pi'$ in the first $i$ positions and the same as $\pi$ in the rest, i.e. $\overline{\pi} = (\pi(1),\ldots,\pi(i),\pi'(i+1),\ldots,\pi'(\numel))$ and $\overline{\pi}' = (\pi'(1),\ldots,\pi'(i),\pi(i+1),\ldots,\pi(\numel))$. Notice that because of the definition of indistinguishability, $\overline{\pi}\neq\overline{\pi}'$ and $\{\overline{\pi},\overline{\pi}'\}\cap\{\pi,\pi'\} = \emptyset$.

{We define the following alternate choice model 
{probabilities $\overline{p}(\cdot)$:  $\overline{p}(\overline{\pi}) = \overline{p}(\overline{\pi}') = q, \overline{p}(\pi') = q'-q$, and $\overline{p}(\pi) = 0$. For all other rankings $\pi''\in \Pi$, we set $\overline{p}(\pi'') = p(\pi'')$. As a result,  depending on whether $q'-q < \kappa$ or $q'-q \geq \kappa$, we obtain that $\overline{\Pi}^F = (\Pi^F\setminus \SET{\pi, \pi'}) \cup \SET{\overline{\pi},\overline{\pi}'}$, or $\overline{\Pi}^F = (\Pi^F\setminus \SET{\pi}) \cup \SET{\overline{\pi},\overline{\pi}'}$.} }

We will prove that no algorithm --- regardless of computational power or number of queries --- can distinguish between $\Pi^F$ and $\overline{\Pi}^F$. We do so by a simple coupling argument: we couple the draws of types from $\Pi^F$ and $\overline{\Pi}^F$ in such a way that the algorithm will see the exact same responses for each query. Importantly, because the algorithm chooses the query \emph{without knowledge of the consumer's type}, the coupling can be based on the query.

The coupling is straightforward. For any type $\pi'' \in \Pi^F \setminus \{ \pi, \pi' \}$, the type $\pi''$ is drawn under $\Pi^F$ exactly when it is drawn under $\overline{\Pi}^F$. In particular, exactly the same item is chosen by the drawn consumer under both models, because the consumer has the same type.

For the remaining types $\pi, \pi', \overline{\pi}, \overline{\pi}'$, we use the following couplings, depending on the query set $\offerset$:
\begin{itemize}[leftmargin=*]
    \item If $\offerset \cap \toppi \neq \emptyset$, then the coupling produces $(\pi,\overline{\pi})$ with probability $q$, $(\pi', \overline{\pi}')$ with probability $q$, and $(\pi',\pi')$ with probability $q'-q$. That is, for example, when $\pi$ is generated under $\Pi^F$ with probability $q$, $\overline{\pi}$ is also generated under $\overline{\Pi}^F$ with the same probability.%, and under both types (i.e., $\pi$ and $\overline{\pi}$), the responses to $\offerset$ are the same. (Recall that $S\cap B\neq \emptyset$.)
    
    Notice that the respective marginal probabilities are all as prescribed by the choice models, so this is a valid coupling. Furthermore, because only an item from $\offerset \cap\toppi$ can be chosen, the same item will be chosen for each possible pair of types $(\pi,\overline{\pi}), (\pi',\overline{\pi}'),(\pi',\pi')$, because for each pair, the two types comprising it agree on the ranking of the top $i$ items.
    
    \item If $\offerset \cap \toppi = \emptyset$, then the coupling produces $(\pi, \overline{\pi}')$ with probability $q$, $(\pi', \overline{\pi})$ with probability $q$, and $(\pi',\pi')$ with probability $q'-q$. Again, notice that all marginal probabilities match those prescribed by the choice models, so the coupling is valid.
    
    Because the chosen item will always be from $\offerset \setminus \toppi$, all possible pairs of types will choose the same item; this is because each such pair has the same rankings for $\offerset \setminus \toppi$.
\end{itemize}

We have thus produced a valid coupling under which exactly the same item is chosen under $\Pi^F$ and $\overline{\Pi}^F$. In particular, this means that the marginal probabilities of choosing each item are the same under both choice models, so an algorithm will observe the exact same distribution under both models, and cannot distinguish them.
\end{proof}

\subsection{Proof of Proposition~\ref{prop:exact_choice_prob}}
\label{sec:proof:prop:exact_choice_prob}

{\constructingchoice*}

{\begin{proof} 
To prove the first identity (that $\choicetype{z}{S} = \bigcup_{A: A \cap S = \emptyset} \edgetype{A}{z}$), we prove both inclusions. First, any type $\pi$ choosing $z$ from $S$ ranks some set $A \subseteq \universe \setminus S$ ahead of $z$, i.e., belongs to some $\edgetype{A}{z}$ with $A \cap S = \emptyset$. Conversely, every type $\pi \in \edgetype{A}{z}$ ranks $z$ ahead of all other items in $S$, and so chooses $z$ from $S$.
   
To prove the probability identity, we apply the definition of $q_z(S)$, then the set identity we just proved, then the disjointness of the sets $\edgetype{A}{z}$ due to Proposition~\ref{prop:basic-properties}, and finally the definition of $\edgeprob{A}{z}$, to obtain that
  \[ 
   q_z(S) = p(\choicetype{z}{S}) 
   = p \left(\bigcup_{A: A \cap S = \emptyset} \edgetype{A}{z} \right)
   = \sum_{A: A \cap S = \emptyset} p(\edgetype{A}{z})
   = \sum_{A: A \cap S = \emptyset} \edgeprob{A}{z}.
  \]
\end{proof}}

\subsection{Proof of Theorem~\ref{thm:construct_exact_DAG}}\label{sec:proof:thm:construct_exact_DAG}

{\constructingexact*}

{\begin{proof} 
Correctness follows directly by induction on the levels, showing that $\hatp(A) = p(A)$ for all $A$, and $\hatedgeprob{A}{z} = \edgeprob{A}{z}$ for all $A, z$. The induction step uses Lemma~\ref{lem:inter_2} and Equation~\eqref{eqn:flow-equation}.

For the running time (and number of oracle queries), note that both the number of nodes at any level $j$ as well as the number of edges from level $j$ to $j+1$ are upper-bounded by $T$.\footnote{This upper bound could be quite loose. It is easy to construct cases in which the number of nodes is linear in $n$, while the number of types is exponential. In fact, this type of example illustrates the value of the DAG representation in compressing the model's description, and thereby making it more intuitively understandable.}
To find all edges emanating from level $j$, the algorithm checks at most $nT$ candidates (each possible item to add). Thus, the total computation at each level is at most $O(n T)$, and because there are $n$ levels, the total computation takes $O(n^2 T)$.
\end{proof}}

\section{Set Cover and the Greedy Algorithm} \label{sec:greedy}
The minimum set covering problem we define in Section~\ref{sec:AlgIE_new} is equivalent to the classic \textsc{Set Cover} problem. In the classic \textsc{Set Cover} problem, a universe $U = \SET{1,2,\ldots,n}$ is given along with a collection $\collectsets$ of $m$ sets whose union covers $U$.
The goal is to find the smallest subcollection of $\collectsets$ whose union still covers $U$.

In our case, the universe is {a subset $\familysetnodepre \subseteq \familysetnodepre(A,z)$} because as can be seen in Lemma~\ref{lem:ie-lemma}, the goal is to find a set cover of {such sets}. Since every element of $\familysetnodepre(A,z)$ is a proper subset of $A$, for each element $A' \in \familysetnodepre$, there always exists an $x$ such that $x\in A\setminus A'$ and $A'\subseteq A\setminus\{x\}$. We remark that the reduction also works in the other direction: every \textsc{Set Cover} instance can be encoded in our setting.
Given an instance $(U, \collectsets)$, we let $j=|\collectsets|$ be the number of sets, and set $A = \SET{1, 2, \ldots, j}$. Then, for each element $u \in U$, let $T_u$ be the indices of sets $S_i \in \collectsets$ such that $u \in S_i$. We define a corresponding prefix $A'_u = A \setminus T_u$. Because $T_u \neq \emptyset$ for all $u$ (otherwise, $u$ could never be covered), all $A'_u$ are strict subsets of $A$. Furthermore, $A'_u \subseteq A \setminus \SET{x}$ for exactly those $x$ such that $u \in S_x$; thus, the instance we generated exactly encodes the original \textsc{Set Cover} instance.

The preceding reduction from \textsc{Set Cover} to our problem shows that both the NP-hardness and approximation hardness to within better than $(1-o(1)) \cdot \ln n$ (\cite{feige1998threshold}) carry over to our problem.
On the positive side, the reduction from our problem \emph{to} \textsc{Set Cover} shows that approximation algorithms for the latter can be applied to our problem with the same guarantees.
In particular, this applies to the well known greedy approximation algorithm (see, e.g., \citep{kleinberg2006algorithm}), which has a multiplicative approximation guarantee of $\ln n$.\footnote{There are other known approximation algorithms for \textsc{Set Cover} with the same guarantee.}  

The greedy algorithm for \textsc{Set Cover}, denoted by $\greedymincover\left(U,\collectsets\right)$, repeatedly adds a set $L \in \collectsets$ which covers the largest number of elements from $U$ which had not been covered by previously added sets. It terminates when all of $U$ is covered. Specialized to our setting, the greedy algorithm $\greedymincover\left(\familysetnodepre,\familyminus(A)\right)$  repeatedly selects a set from $\familyminus(A)$ that contains as subsets the largest number of prefixes $A' \in \familysetnodepre$ not contained in any selected set so far. The algorithm terminates when all prefixes in $\familysetnodepre$ have been covered.

%%%%%%%%%%%%%%%%%%%%%%%%%%%%%%%%%%%%%%%%%%%%%%%%%%%%%%%%%%%%%%%%%%%%%%%%%%%%%%%%%%%%%%%%%%
%%%%%%%%%%%%%%%%%%%%%%%%%%%%%%%%%%%%%%%%%%%%%%%%%%%%%%%%%%%%%%%%%%%%%%%%%%%%%%%%%%%%%%%%%%
%%%%%%%%%%%%%%%%%%%%%%%%%%%%%%%%%%%%%%%%%%%%%%%%%%%%%%%%%%%%%%%%%%%%%%%%%%%%%%%%%%%%%%%%%%
\section{Proof of Proposition~\ref{prop:incl-excl-exact}{, and a More General Inclusion-Exclusion Lemma}} \label{sec:Proposition:proof} 

Proposition~\ref{prop:incl-excl-exact} follows as an immediate corollary of the following lemma.
(The more general lemma {is} also central to our analysis of \AlgDAG.)

\begin{restatable}{lemma}{inclusionexclusion}
\label{lem:ie-lemma}
  Fix a set $A$ and item $z \notin A$.
  Let $\familysetnodepre \subseteq \familysetnodepre(A,z)$, and $\mathcal{C}$ be a set cover of $\familysetnodepre$ using $\familyminus(A)$, with the additional property that $\familysetnodepre$ is maximal, in the sense that $\mathcal{C}$ is \emph{not} a set cover of any $\familysetnodepre' \supsetneq \familysetnodepre, \familysetnodepre' \subseteq \familysetnodepre(A,z)$.
  Then,
   \[
     \sum_{A' \in \familysetnodepre} \edgeprob{A'}{z} 
     = \sum_{\mathcal{B} \subseteq \mathcal{C}, \mathcal{B} \neq \emptyset} (-1)^{|\mathcal{B}| + 1} \cdot q_z(\universe \setminus \bigcap_{C \in \mathcal{B}} C).
   \]
\end{restatable}

\begin{proof}
We begin by observing that because $\mathcal{C}$ is a set cover of $\familysetnodepre$, every $A' \in \familysetnodepre$ is contained in some $C \in \mathcal{C}$.
As a result, $\edgetype{A'}{z} \subseteq \choicetype{z}{C}$; this is because any ranking in $\edgetype{A'}{z}$ prefers $z$ over all of $\universe \setminus A' \supseteq \universe \setminus C$, and will therefore choose $z$ from $\universe \setminus C$.
We have thus shown that $\bigcup_{A' \in \familysetnodepre} \edgetype{A'}{z} \subseteq \bigcup_{C \in \mathcal{C}} \choicetype{z}{\universe \setminus C}$.
For the converse inclusion, observe that for each type $\pi \in \choicetype{z}{\universe \setminus C}$, the set of items $A'$ ranked ahead of $z$ in $\pi$ must be a subset of $C$.
Therefore, $\pi \in \edgetype{A'}{z}$ for this $A'$.
Because $A' \in \familysetnodepre(A,z)$, my maximality of $\familysetnodepre$, we must have $A' \in \familysetnodepre$ as well.
Thus, we have shown the converse inclusion $\bigcup_{A' \in \familysetnodepre} \edgetype{A'}{z} \supseteq \bigcup_{C \in \mathcal{C}} \choicetype{z}{\universe \setminus C}$.
Applying that the sets $\edgetype{A'}{z}$ are all disjoint by Proposition~\ref{prop:basic-properties}, we obtain that 
\[
  \sum_{A' \subsetneq A} \edgeprob{A'}{z} 
  = p\left(\bigcup_{A' \subsetneq A} \edgetype{A'}{z}\right)
  = p\left(\bigcup_{C \in \mathcal{C}} \choicetype{z}{\universe \setminus C}\right).
\]
In the remainder of the proof, we will show that $p\left(\bigcup_{C \in \mathcal{C}} \choicetype{z}{\universe \setminus C}\right)$ equals the Inclusion-Exclusion sum.
For any non-empty subset $\mathcal{B} \subseteq \mathcal{C}$, we write 
\[
S_{\mathcal{B}} 
\; = \; \universe \setminus \bigcap_{C \in \mathcal{B}} C
\; = \; \bigcup_{C \in \mathcal{B}} (\universe \setminus C).
\]
This implies that $S_{\mathcal{B} \cup \mathcal{B}'} = S_{\mathcal{B}} \cup S_{\mathcal{B'}}$.
Because $z \notin C$ for all $C \in \mathcal{C}$ (because $\mathcal{C} \subseteq \familyminus(A)$, and $z \notin A$), we obtain that $z \in S_{\mathcal{B}}$. 
%Recall that $\Theta_z (S_{\mathcal{B}})$ is the set of types choosing $z$ when offered $S_{\mathcal{B}} \ni z$. 
A type $\pi$ chooses $z$ from $\bigcup_{C \in \mathcal{B}} (\universe \setminus C)$ if and only if $\pi$ chooses $z$ from $\universe \setminus C$ for each $C \in \mathcal{B}$. 
Therefore, $\Theta_z (S_{\mathcal{B}}) = \bigcap_{C \in \mathcal{B}} \Theta_z (\universe \setminus C)$.
Applying the standard Inclusion-Exclusion formula to the sets $\Theta_z (\universe \setminus C)$, we obtain that
\begin{align*}
 p \left( \bigcup_{C \in \mathcal{C}} \Theta_z (\universe \setminus C) \right)
 & = \sum_{\mathcal{B} \subseteq \mathcal{C}, \mathcal{B} \neq \emptyset} (-1)^{|\mathcal{B}|+1} \cdot p \left( \bigcap_{C \in \mathcal{B}} \Theta_z (\universe \setminus C) \right)
\\ & = \sum_{\mathcal{B} \subseteq \mathcal{C}, \mathcal{B} \neq \emptyset} (-1)^{|\mathcal{B}|+1} \cdot p ( \Theta_z (S_{\mathcal{B}}))
\\ & = \sum_{\mathcal{B} \subseteq \mathcal{C}, \mathcal{B} \neq \emptyset} (-1)^{|\mathcal{B}|+1} \cdot q_z \left( \universe \setminus \bigcap_{C \in \mathcal{B}} C \right),
\end{align*}

where the last step used that $q_z(S) = p(\Theta_z (S))$ for all sets $S \ni z$.
\end{proof}

\begin{extraproof}{Proposition~\ref{prop:incl-excl-exact}}
  Applying Lemma~\ref{lem:ie-lemma} with $\familysetnodepre = \familysetnodepre(A,z)$ gives us that 
  $\sum_{A' \subsetneq A} \edgeprob{A'}{z} = \sum_{\mathcal{B} \subseteq \mathcal{C}, \mathcal{B} \neq \emptyset} (-1)^{|\mathcal{B}| + 1} \cdot q_z(\universe \setminus \bigcap_{C \in \mathcal{B}} C)$.
  Substituting this identity into Lemma~\ref{lem:inter_2} and rearranging now completes the proof.
\end{extraproof}

%%%%%%%%%%%%%%%%%%%%%%%%%%%%%%%%%%%%%%%%%%%%%%%%%%%%%%%%%%%%%%%%%%%%%%%%%%%%%%%%%%%%%%%%%%
%%%%%%%%%%%%%%%%%%%%%%%%%%%%%%%%%%%%%%%%%%%%%%%%%%%%%%%%%%%%%%%%%%%%%%%%%%%%%%%%%%%%%%%%%%
%%%%%%%%%%%%%%%%%%%%%%%%%%%%%%%%%%%%%%%%%%%%%%%%%%%%%%%%%%%%%%%%%%%%%%%%%%%%%%%%%%%%%%%%%%
\section{Proof of Theorem~\ref{thm:main}}
\label{sec:proof}
In this section, we give a proof of Theorem~\ref{thm:main}, with the proofs of two technical lemmas deferred to {separate appendices below}.

The main part of the analysis is captured by Lemma~\ref{lemma:ie-proof}, which will be stated later in this section.  This lemma shows that under the general model, the estimates $\hatedgeprob{A}{z}$ which \AlgIE obtains for the edge probabilities $\edgeprob{A}{z}$ are accurate enough with sufficiently high probability. The lemma further  bounds  the number of queries required under the \randmodel{\kappa}{\raresum}. 
The number of queries required increases exponentially in the size of the set covers used, so to bound the former, we would like to bound the latter. We observe that {only} certain configurations of types could possibly cause a need for large set covers. Therefore, we first show that under the \randmodel{\kappa}{\raresum}, these configurations are sufficiently unlikely to occur. 

More formally, for any positive integer $c\in [n-1]$, we define $\Event{c}$ to be the event that there are two distinct {frequent} types $\pi_1 \neq \pi_2 \in \Pi^F$, a position $j \leq \numitems$, and an item $z$ in position $j+1$ of ranking $\pi_1$, such that 
\begin{itemize}
    \item $z$ is in position $c+1$ or later in $\pi_2$, and
    \item the set of items preceding $z$ in $\pi_2$ is a subset of the first $j$ items in $\pi_1$.
\end{itemize}
Lemma~\ref{lemma:upperboundcover}, proved in Appendix~\ref{sec:upperboundcoverproof}, shows that $\Event{c}$ is unlikely to occur.

\begin{restatable}{lemma}{maxcovering}
\label{lemma:upperboundcover}
Under the \randmodel{\kappa}{\raresum}, when $c \geq \log_{1/\alpha} \left(\frac{K^2 \numel}{\delta}\right)$, the probability of $\Event{c}$ is at most $\Prob{\Event{c}} \leq \delta$.
\end{restatable}

Because $\Event{c}$ can never happen for $c > \numel$, Lemma~\ref{lemma:upperboundcover} is of interest only in the regime when 
$\log_{1/\alpha} \left(\frac{K^2 \numel}{\delta}\right) \leq \numel$, or $K\leq \sqrt{\frac{(1/\alpha)^\numel \delta}{\numel}}$. Otherwise, the lemma is trivial since the probability of the event $\Event{c}$  is $0$.

Next, we show that having computed the graph correctly up to the previous level is enough to ensure that \AlgIE returns accurate estimates with high probability --- this forms the key of the inductive correctness proof of \AlgDAG. Furthermore, when \Event{\probboundc} happens, all the set cover solutions in \AlgIE are small enough that the number of samples stays polynomial.

\begin{restatable}{lemma}{ietheorem}
\label{lemma:ie-proof}
Let $j \in \SET{0,1,\ldots,\numitems-1}$ be a level, and assume that {$\hatV_j = V^F_j$ and $\hatE_j = E^F_j$}. 
Let $\hatedgeprob{A}{z}$ be the estimated edge probability returned by \AlgIE (Algorithm~\ref{alg:ie}), called with parameters $\epsilon$ and $\delta$. Then, $\hatedgeprob{A}{z}$ satisfies
\begin{equation*}
 \left| \hatedgeprob{A}{z} - \edgeprob{A}{z} \right| \leq \epsilon + \raresum
\end{equation*}
with probability at least $1-\delta$.

Furthermore, if \Event{c} did not happen, then the number of consumer queries and the total computation time are upper-bounded by $O\left(\frac{2^{2c} (c + \log (1/\delta))}{\epsilon^2}\right)$.
\end{restatable}

\begin{extraproof}{Theorem~\ref{thm:main}}
We are now ready to prove Theorem~\ref{thm:main}. 
We will prove by induction on the level $j$ that with probability 
at least $(1-\frac{\delta}{\numitems})^j$, 
the following four all hold:
\begin{align*}
\hatV_j & = V^F_{\numitems} 
& \hatE_j & = E^F_{\numitems} 
& 
\max_{A \in V_j} \left|\hatp(A) - p(A)\right| & \leq \epsilon + \numitems \raresum
& 
\max_{A \in V_j, A \cup \SET{z} \in V_j} \left| \hatedgeprob{A}{z} - \edgeprob{A}{z} \right| \leq \epsilon/\numitems + \raresum.
\end{align*}

The first part of the theorem then follows by setting $j = \numitems = \alpha \numel$, and noting that 
$(1-\frac{\delta}{\numitems})^{\numitems} \geq 1-{\delta}$.

We now complete the induction proof.
The base case $j=0$ holds because $\hatV_0 = \SET{\emptyset} = V_0^F$ and $\hatE_0 = \emptyset = E_0^F$ hold deterministically.
For the induction step, consider a $j < \numitems$.
By induction hypothesis, $\hatV_j = V_j^F$ and $\hatE_j = E_j^F$ with probability at least $(1-\frac{\delta}{\numitems})^j$ (and the estimates of $\bm{p}$ and $\bm{e}$ up to level $j$ are accurate to within an additive $\epsilon + \numitems \raresum$ and $\epsilon + \raresum$, respectively). We condition on this case.

To obtain $\hatG_{j+1}$ from $\hatG_j$, Algorithm~\ref{alg:all} goes through each node $A$ on level $j$ of $\hatG_j$ and all candidate items $z \notin A$, and uses \AlgIE to compute an estimate $\hatedgeprob{A}{z}$ of the true value ${\edgeprob{A}{z}}$. Because $\hatV_j = V^F_j$ and $\hatE_j = E^F_j$, Lemma~\ref{lemma:ie-proof} implies that for each of these calls (characterized by the choice of $z$), \AlgIE (which is called by \AlgDAG with parameters $\epsilon' = \min(\epsilon/\numitems, \kappa/4)$ and $\delta' = \frac{\delta}{\alpha \numel^2 K}$) returns an estimate $\hatedgeprob{A}{z}$ that is at most 
$\epsilon' + {\raresum} 
= \min(\epsilon/\numitems, \kappa/4) + \raresum$ far from the true value $\edgeprob{A}{z}$, with probability at least $1-\delta' = 1-\frac{\delta}{\alpha \numel^2 K}$. 

Because there are at most $K$ frequent types, and the nodes on level $j$ of $\hatG_j$ correspond to disjoint sets of types, there are at most $K$ nodes on level $j$. 
For each such node, the algorithm checks at most $\numel$ items $z$ to add. By a union bound over all of the invocations of \AlgIE for these different pairs $(A, z)$, all estimates $\hatedgeprob{A}{z}$ of $\edgeprob{A}{z}$ are simultaneously accurate to within an additive $\epsilon'+\raresum$ with probability at least 
\begin{align*}
\left(1-\frac{\delta}{\alpha \numel^2 K}\right)^{nK} 
\geq 1- (\numel K) \cdot \frac{\delta}{\alpha \numel^2 K}
& = 1-\frac{\delta}{\alpha \numel}
\; = \; 1-\frac{\delta}{\numitems}.
\end{align*}

Under the assumptions of the theorem, $\epsilon' + \raresum \leq \kappa/4 + \raresum < \kappa/2$. In particular, when the estimates are accurate to within $\epsilon' + \raresum$, rare rankings alone can never make the algorithm add a node $A \cup \SET{z}$ to layer $j+1$. 
Conversely, every frequent type has probability at least $\kappa$, so when the estimates are accurate to within $\epsilon' + \raresum$, every frequent type whose first $j+1$ items are ${A} \cup \SET{z}$ has estimated frequency at least $\kappa - (\epsilon' + \raresum) \geq \kappa - (\kappa/4+\raresum) \geq \kappa/2$, hence causes the creation of the node $A \cup \SET{z}$. 
Thus, we have shown that $\hatV_{j+1} = V^F_{j+1}$ and $\hatE_{j+1} = E^F_{j+1}$ whenever all estimates are accurate to within $\epsilon' + \raresum$.

To bound the error in the estimated probabilities, first observe that the bound on $\left| \hatedgeprob{A}{z} - \edgeprob{A}{z} \right|$ follows directly because $\epsilon' + \raresum \leq \epsilon/\numitems + \raresum$. Now, consider a node $A \in V_{j+1}^F$. 
The true combined probability of types who rank all of $A$ ahead of all items not in $A$ is $p(A) = \sum_{z \in A} \edgeprob{A \setminus \SET{z}}{z}$.
\AlgDAG, on the other hand, uses estimated probabilities $\hatedgeprob{A \setminus \SET{z}}{z}$ in place of the true probabilities, and furthermore only adds these estimates for elements $z$ such that the set $A \setminus \SET{z} \in V^F_j$ is a frequent type, i.e., it uses
$\hatp(A) = \sum_{z \in A: A \setminus \SET{z} \in V^F_j} \hatedgeprob{A \setminus \SET{z}}{z}$.
Using the triangle inequality along with the upper bound of $\epsilon' + \raresum$ on individual estimation errors, and the fact that $|A| \leq \numitems -1$, we get that the error is at most 
\begin{align*}
|p(A)-\hatp(A)| 
&= \Big|
\sum_{z \in A} \edgeprob{A \setminus \SET{z}}{z}
- \sum_{z \in A: A \setminus \SET{z} \in V^F_j} \hatedgeprob{A \setminus \SET{z}}{z} \Big|
\\ & \leq 
\sum_{z \in A: A \setminus \SET{z} \in V^F_j} 
\left| \edgeprob{A \setminus \SET{z}}{z}
- \hatedgeprob{A \setminus \SET{z}}{z} \right|
+ \sum_{z \in A: A \setminus \SET{z} \notin V^F_j} 
\edgeprob{A \setminus \SET{z}}{z}
\\ & \leq 
\sum_{z \in A: A \setminus \SET{z} \in V^F_j} 
\left( \min(\epsilon/\numitems, \kappa/4) + \raresum \right)
+ \raresum
\\ & \leq \numitems \cdot \frac{\epsilon}{\numitems} + 
{(\numitems-1) \cdot \raresum + \raresum}
\\ & = \epsilon + \numitems \cdot \raresum.
\end{align*}

Here, the inequality on the third line used the fact that the last sum adds probabilities that can only arise from disjoint sets of rare types.

Finally, by induction hypothesis, the event we conditioned on (that $\hatV_j = V_j^F$ and $\hatE_j = E_j^F$ and the estimates up to level $j$ are accurate to within an additive $\epsilon + \numitems \raresum$) had probability at least $(1-\frac{\delta}{\numitems})^j$, and conditioned on this event, level $j+1$ is correct (and accurate to within the claimed additive errors) with probability at least $(1-\frac{\delta}{\numitems})$. Hence, the entire inferred DAG representation up to level $j+1$ is accurate with probability at least $(1-\frac{\delta}{\numitems})^{j+1}$, completing the induction step, and thus the proof of the first part of the theorem.

\medskip

We now prove the second part of the theorem, bounding the number of queries and the computation time under the random model. 
Define $c = \left \lceil \log_{1/\alpha} \left( \frac{K^2 \numel}{\delta} \right) \right \rceil$. 
By Lemma~\ref{lemma:upperboundcover}, with probability at least $1-\delta$, the event $\Event{c}$ did not happen. By a union bound with the first part of the proof, we obtain that with probability at least $1-2\delta$, $G^F_{\numitems}$ is correctly reconstructed, all probability estimates are accurate to within the claimed additive errors, and the event \Event{c} did not happen.

By the second part of Lemma~\ref{lemma:ie-proof}, each invocation of \AlgIE results in $O\left(\frac{K^{4\nu}\cdot n^{2\nu}\cdot\log(\numel K/\delta)}{\delta^{2\nu}\cdot\min(\epsilon,\kappa)^2}\right)$
%$O\left(\frac{\probboundc^2 \log (\probboundc/\delta') }{(\epsilon'-\raresum)^2}\right)$ 
consumer queries and computation. Each frequent type contributes at most one distinct node $A$ for each level $j$, and for each such node $A$, \AlgDAG considers all items $z \notin A$ for addition, and calls \AlgIE. Thus, the total number of calls to \AlgIE is at most $K \numel \numitems$. Substituting, bounding that $\numitems \leq \numel$, and writing the constant $\cons = \log_{\alpha} (1/2) > 0$, we obtain that the number of consumer queries is at most
\begin{align*}
O\left(K \numel \numitems \cdot 
\frac{2^{2c} (c + \log (2\alpha \numel^2 K/\delta))}{ \min(\epsilon/\numitems,\kappa/4)^2}\right)
%\frac{\probboundc^2 \log (2\probboundc \alpha \numel^2 K/\delta) }{(\frac{\epsilon-\raresum}{\numitems}-\raresum)^2}
& = O\left( \frac{K^{1+4\cons} \cdot \numel^{4+2\cons} \cdot 
\log (\numel K/\delta)}{\delta^{2\cons} \cdot \min(\epsilon,\kappa)^2}\right),
\end{align*}
which is polynomial. Since the computation is dominated by the queries, the same bound applies to the computation time.
\end{extraproof}

\section{Proof of Lemma~\ref{lemma:upperboundcover}}
\label{sec:upperboundcoverproof}
In this section, we prove Lemma~\ref{lemma:upperboundcover}, which is restated below.

\maxcovering*

\begin{proof}
We want to bound the probability that there are two distinct types $\pi_1 \neq \pi_2 \in \Pi^F$ and a position $j \leq \numitems$ such that the item $z$ which is in position $j+1$ in $\pi_1$ is in position $c+1$ or later in $\pi_2$, and the set of items preceding $z$ in $\pi_2$ is a subset of the first $j$ items in $\pi_1$.

For now, focus on two types and a fixed position $j \leq \numitems$, with $\pi_1$ already drawn (and thus defining $z$), and the uniformly random draw defining $\pi_2$. Let \Event{c,k} (for $k \geq c+1$) denote the event that $z$ is in position $k$ in $\pi_2$, and the first $k-1$ items of $\pi_2$ are all among the first $j$ items of $\pi_1$.

There are $\binom{j}{k-1}$ ways to pick the first $k-1$ items of $\pi_2$ from the first $j$ items of $\pi_1$, $(k-1)!$ ways to order them, and $(\numel-k)!$ ways to order the items after position $k$. Since there are $\numel!$ total rankings, the probability of \Event{c,k} is
\begin{align*}
\Prob{\Event{c,k}}
& = \frac{\binom{j}{k-1} \cdot (k-1)! \cdot (\numel-k)!}{\numel!}
\; = \; \frac{j! \cdot (\numel-k)!}{(j+1-k)! \cdot \numel!}
\; \leq \; \frac{j! \cdot (\numel-(c+1))!}{(j-c)! \cdot \numel!},
\end{align*}
where the inequality holds because the probability is monotone decreasing in $k$ and thus maximized when $k$ is as small as possible, i.e., $k=c+1$.
Next, we bound
%\begin{align*}
\[
\frac{j! \cdot (\numel-(c+1))!}{(j-c)! \cdot \numel!}
\; = \;
\frac{1}{\numel-c} \cdot
\prod_{k=0}^{c-1} \frac{j-k}{\numel-k}
%\\ & 
\; \leq \; \frac{1}{\numel-c} \cdot 
(j/\numel)^c
%\\ & 
\; \leq \;
\frac{1}{\numel-c} \cdot 
(\numitems/\numel)^c,
\]
%\end{align*}
because $(j-k)/(\numel-k)$ is monotone decreasing in $k$, so it is maximized at $k=0$, and $j \leq \numitems$.
Now, taking a union bound over all choices of $k=c+1, \ldots, j$ (of which there are at most $\numel-c$) as well as all choices of $j \leq \numitems \leq \numel$ and ordered pairs of types (of which there are at most $K(K-1) \leq K^2$), we obtain the bound
\begin{align*}
  \Prob{\Event{\probboundc}}
  & \leq
K^2 \cdot \numel \cdot (\numel-c) \cdot
\frac{1}{\numel-c} \cdot 
(\numitems/\numel)^c
%\\ & =
%K^2 \cdot \numel \cdot
%\frac{\numitems-c}{\numel-c} \cdot 
%\alpha^{\minsize}
% \\ & \leq
%K^2 \cdot \numel\cdot \frac{\numitems}{\numel} \cdot 
% (1/\alpha) \cdot \alpha^{\log_2 c}
\; = \;  
K^2 \cdot \numel \cdot \alpha^c.
\end{align*}

%Here, the inequality used that $\minsize \geq \log_2 c - 1$.
Finally, substituting the lower bound on $c$, we obtain that
\begin{align*}
  \Prob{\Event{\probboundc}}
  & \leq K^2 \cdot \numel \cdot \alpha^{c}
  \; \leq \; K^2 \cdot \numel \cdot \frac{\delta}{K^2 \cdot \numel}
  \; = \; \delta,
\end{align*}
completing the proof.
\end{proof}

%%%%%%%%%%%%%%%%%%%%%%%%%%%%%%%%%%%%%%%%%%%%%%%%%%%%%%%%%%%%%%%%%%%%%%%%%%%%%%%%%%%%%%%%%%
%%%%%%%%%%%%%%%%%%%%%%%%%%%%%%%%%%%%%%%%%%%%%%%%%%%%%%%%%%%%%%%%%%%%%%%%%%%%%%%%%%%%%%%%%%
%%%%%%%%%%%%%%%%%%%%%%%%%%%%%%%%%%%%%%%%%%%%%%%%%%%%%%%%%%%%%%%%%%%%%%%%%%%%%%%%%%%%%%%%%%

\section{Proof of Lemma~\ref{lemma:ie-proof}}
\label{sec:IElemmaproof}

In this section, we complete the proof of Lemma~\ref{lemma:ie-proof}, which is restated below for convenience.

\ietheorem*

\begin{extraproof}{Lemma~\ref{lemma:ie-proof}}
First, by assumption of the lemma, we have that $\hatV_j = V^F_j$ and $\hatE_j = E^F_j$, which implies that $\familysetnodeprealg = \familysetnodepreF(A,z)$. is the set of all prefixes $A' \subsetneq A$ such that at least one frequent type has $A'$ in the first $|A'|$ positions, followed by $z$.
Recall that \AlgIE computes an approximately smallest set cover $\mathcal{C}$ of $\familysetnodeprealg = \familysetnodepreF(A,z)$ using sets in $\familyminus(A)$; specifically, this means that $\mathcal{C} \subseteq \familyminus(A)$ and that for each set $A' \in \familysetnodepreF(A,z)$, there exists a set $C \in \mathcal{C}$ with $A' \subseteq C$.

By Lemma~\ref{lem:inter_2}, the true probability of having a prefix of $A$ followed by $z$ is
\begin{align*}
\edgeprob{A}{z}
    & = q_z(\universe \setminus A) 
- \sum_{A' \subsetneq A} \edgeprob{A'}{z}.
\end{align*}

Let $\familysetnodepre \supseteq \familysetnodepreF(A,z), \familysetnodepre \subseteq \familysetnodepre(A,z)$ be maximal such that $\mathcal{C}$ is a set cover of $\familysetnodepre$.
Then, by Lemma~\ref{lem:ie-lemma}, applied to $\familysetnodepre$, we have that
\[ 
  \sum_{A' \in \familysetnodepre} \edgeprob{A'}{z}      
  = \sum_{\mathcal{B} \subseteq \mathcal{C}, \mathcal{B} \neq \emptyset} (-1)^{|\mathcal{B}| + 1} \cdot q_z(\universe \setminus \bigcap_{C \in \mathcal{B}} C).
\]
Substituting this identity into the previous identity for $\edgeprob{A}{z}$, we obtain that
\begin{align*}
\Big| \edgeprob{A}{z}
- \Big( q_z(\universe \setminus A) 
- \sum_{\mathcal{B} \subseteq \mathcal{C}, \mathcal{B} \neq \emptyset} (-1)^{|\mathcal{B}| + 1} \cdot q_z(\universe \setminus \bigcap_{C \in \mathcal{B}} C) \Big) \Big|
& =
\sum_{A' \subsetneq A, A' \notin \familysetnodepre} \edgeprob{A'}{z}
%\leq \sum_{A' \in \familysetnodepreR(A,z)} \edgeprob{A'}{z}.
\\ & = p \Big(\bigcup_{A' \subsetneq A, A' \notin \familysetnodepre} \edgetype{A'}{z} \Big)
\\ & \leq \raresum.
\end{align*}
Here, the second step used the disjointness of the $A'$ from Proposition~\ref{prop:basic-properties}, and the final step used that every type $\pi \in \edgetype{A'}{z}$ with $A' \subsetneq A, A' \notin \familysetnodepre$ must be a rare type, because $\familysetnodepre \supseteq \familysetnodepre^F(A,z)$.

\AlgIE, on the other hand, uses the estimated $\hatq_z$ values to compute an estimate of the edge probability   
\begin{align*}
    \hatedgeprob{A}{z} 
   & = \hatq_z(\universe \setminus A) 
- \sum_{\mathcal{B} \subseteq \mathcal{C}, \mathcal{B} \neq \emptyset} (-1)^{|\mathcal{B}| + 1} \cdot \hatq_z(\universe \setminus \bigcap_{C \in \mathcal{B}} C).
\end{align*}

Thus, the absolute estimation error is 
\begin{align*}
\left| \edgeprob{A}{z}
    -  \hatedgeprob{A}{z} \right|
& = \Big| 
\edgeprob{A}{z}
- \Big(q_z(\universe \setminus A) 
- \sum_{\mathcal{B} \subseteq \mathcal{C}, \mathcal{B} \neq \emptyset} (-1)^{|\mathcal{B}| + 1} \cdot q_z(\universe \setminus \bigcap_{C \in \mathcal{B}} C) \Big)
\\ & \quad + \Big(q_z(\universe \setminus A) 
- \sum_{\mathcal{B} \subseteq \mathcal{C}, \mathcal{B} \neq \emptyset} (-1)^{|\mathcal{B}| + 1} \cdot q_z(\universe \setminus \bigcap_{C \in \mathcal{B}} C) \Big)
\\ & \quad - \Big(
   \hatq_z(\universe \setminus A) 
    - \sum_{\mathcal{B} \subseteq \mathcal{C}, \mathcal{B} \neq \emptyset} (-1)^{|\mathcal{B}| + 1} \cdot 
     \hatq_z(\universe \setminus \bigcap_{C \in \mathcal{B}} C) \Big) \Big|
\\ & \leq         
  \Big| \edgeprob{A}{z}
     - \Big(q_z(\universe \setminus A) 
     - \sum_{\mathcal{B} \subseteq \mathcal{C}, \mathcal{B} \neq \emptyset} (-1)^{|\mathcal{B}| + 1} \cdot q_z(\universe \setminus \bigcap_{C \in \mathcal{B}} C) \Big) \Big|
\\ & \quad + \left| q_z(\universe \setminus A) 
       - \hatq_z(\universe \setminus A) \right|
   + \sum_{\mathcal{B} \subseteq \mathcal{C}, \mathcal{B} \neq \emptyset} 
       \left| q_z(\universe \setminus \bigcap_{C \in \mathcal{B}} C)
       - \hatq_z(\universe \setminus \bigcap_{C \in \mathcal{B}} C) \right|
\\ & \leq         
  \raresum
  + \left| q_z(\universe \setminus A) 
       - \hatq_z(\universe \setminus A) \right|
%\\ & \quad 
+ \sum_{\mathcal{B} \subseteq \mathcal{C}, \mathcal{B} \neq \emptyset} 
       \Big| q_z(\universe \setminus \bigcap_{C \in \mathcal{B}} C)
       - \hatq_z(\universe \setminus \bigcap_{C \in \mathcal{B}} C) \Big|,
\end{align*}
where the first inequality is based on applying the triangle inequality first at the outer level, then pulling it through the sum.

The key step for the analysis is to upper-bound the differences 
$|q_z(\universe \setminus S) - \hatq_z(\universe \setminus S)|$,
where either $S=A$ or $S = \bigcap_{C \in \mathcal{B}} C$ for some $\mathcal{B} \subseteq \mathcal{C}$.
We do this using Hoeffding's Inequality (Lemma~\ref{lemma:Hoeffding}). 
Specifically, $\hatq_z(\universe \setminus S)$ is the average of at least
\[ 
\numsamples = \left\lceil \frac{2^{2|\mathcal{C}|-1} \cdot (\ln(1/\delta) + (|\mathcal{C}|+1) \cdot \ln(2))}{\epsilon^2} \right\rceil 
\]
i.i.d.~Bernoulli random variables, with probability $q_z(\universe \setminus S)$ of taking the value 1. Hence, by Hoeffding's Inequality for $\numsamples$ Bernoulli random variables with mean $q_z(\universe \setminus S)$, and by setting  $\tau = \frac{\epsilon}{2^{|\mathcal{C}|}}$, we have
\begin{align*}
\Prob{|q_z(\universe \setminus S) - \hatq_z(\universe \setminus S)| > \tau}
& \leq 2 \exp(-2\tau^2 \numsamples)
\\ & = 2 \exp\left(-2\left(\frac{\epsilon}{2^{|\mathcal{C}|}}\right)^2 \left\lceil \frac{2^{2|\mathcal{C}|-1} \cdot (\ln(1/\delta) + (|\mathcal{C}|+1) \cdot \ln(2))}{\epsilon^2} \right\rceil  \right)
\\ & \leq
\frac{\delta}{2^{|\mathcal{C}|}}.
\end{align*}

There are $2^{|\mathcal{C}|} - 1$ sets $S$ of the form $S = \bigcap_{C \in \mathcal{B}} C$ (one for each $\mathcal{B} \subseteq \mathcal{C}, \mathcal{B} \neq \emptyset$), plus the set $S = A$. Hence, by a union bound over all such sets $S$, with probability at least 
$1 - 2^{|\mathcal{C}|} \cdot \frac{\delta}{2^{|\mathcal{C}|}} = 1 - \delta$,
all of the estimates satisfy 
$|q_z(\universe \setminus S) - \hatq_z(\universe \setminus S)| \leq \frac{\epsilon}{2^{|\mathcal{C}|}}$. 
In that case, the total estimation error is at most 
\begin{align*}
| \edgeprob{A}{z}
    - \hatedgeprob{A}{z}) |
& \leq \raresum +
    \left| q_z(\universe \setminus A) 
       - \hatq_z(\universe \setminus A) \right|
   + \sum_{\mathcal{B} \subseteq \mathcal{C}, \mathcal{B} \neq \emptyset} 
       \left| q_z(\universe \setminus \bigcap_{C \in \mathcal{B}} C)
       - \hatq_z(\universe \setminus \bigcap_{C \in \mathcal{B}} C) \right|
\\ & \leq \raresum +  2^{|\mathcal{C}|} \cdot 
     \frac{\epsilon}{2^{|\mathcal{C}|}}
\\ & = \raresum + \epsilon.  
\end{align*}

Finally, we prove the second part of the lemma, i.e., the bound on the number of consumer queries (and hence computation, since the inclusion-exclusion computation dominates the computing time as well). We will show that whenever \Event{c} did not happen, $|\mathcal{C}| \leq c$. 
For contradiction, assume that $|\mathcal{C}| > c$, i.e., the greedy min cover algorithm returned a solution of size at least $c+1$. Recall that the greedy algorithm selects subsets of the form $A \setminus \SET{x}$ until all prefixes are covered, and that by the assumption of the lemma that $\hatG_j = G^F_j$, the set of prefixes that the greedy algorithm is trying to cover is exactly $\familysetnodepreF(A,z)$.
Consider the moment when the greedy algorithm has selected $c$ sets $C_i = A \setminus \SET{x_i}$ in its cover. At that point, there is still an uncovered set $A' \in \familysetnodepreF(A,z)$, since the greedy algorithm added at least one more set. Because $A'$ is not contained in any of the $C_i$, it must contain all of the elements $x_i$; in particular, it has size at least $c$. Furthermore, by definition of $\familysetnodepreF(A,z)$, $A'\subseteq A$, and $A'$ is a prefix followed by $z$. Therefore, there exist a frequent type $\pi_1$ which has all of $A$ followed by $z$ and a type $\pi_2$ which has all of $A'$ preceding $z$ directly, where $A'\subseteq A$ and $|A'|\geq c$, implying that $z$ is in position $c+1$ or later in $\pi_2$. This certifies that the event $\Event{c}$ happened, which the assumption of the lemma ruled out.
Substituting the bound $|\mathcal{C}| \leq c$ into the definition of $\numsamples$ gives us the claimed bound.
\end{extraproof}

\begin{lemma}[Hoeffding's Inequality (\cite{10.2307/2282952})]
\label{lemma:Hoeffding}
Let $X_1, X_2, \ldots, X_n$ be independent random variables bounded between $[0,1]$. Then,
\[ \mathbb{P}\left[\left| \sum_i (X_i - \Expect{X_i})\right| > \tau\right] 
\leq 2 \exp(-2\tau^2/n). \]
\end{lemma}

%%%%%%%%%%%%%%%%%%%%%%%%%%%%%%%%%%%%%%%%%%%%%%%%%%%%%%%%%%%%%%%%%%%%%%%%%%%%%%%%%%%%%%%%%%
%%%%%%%%%%%%%%%%%%%%%%%%%%%%%%%%%%%%%%%%%%%%%%%%%%%%%%%%%%%%%%%%%%%%%%%%%%%%%%%%%%%%%%%%%%
%%%%%%%%%%%%%%%%%%%%%%%%%%%%%%%%%%%%%%%%%%%%%%%%%%%%%%%%%%%%%%%%%%%%%%%%%%%%%%%%%%%%%%%%%%

\section{{Additional Details on Experiments}}

{\subsection{Benchmark Implementation}}\label{sec:benchmark}

For both synthetic and sushi  data sets, we do the following for the benchmark: (1) solve the primal LP fully (without constraint sampling) using the \texttt{Gurobi} solver, instead of solving the dual approximately (as done in \cite{farias2013nonparametric}), and (2) impose sparsity on $\bm{\lambda}$, the probability distribution of types. These give the primal LP algorithm an extra benefit{:} solving the primal fully permits searching over a bigger space, and the sparsity constraint helps the LP deal with the indistinguishability problem. Furthermore, we relax the LP constraint to overcome infeasibility, by incorporating $1/\sqrt{\numsamples}$ standard error for the empirically calculated probability that each item in the assortment is chosen, where $\numsamples$ is the number of samples used per offered assortment in \AlgDAG and the benchmark. So instead of having the constraint $\bm{y} = A\bm{\lambda}$, we have 
$\hat{\bm{y}}-1/\sqrt{\numsamples} \cdot \bm{1} \leq A\bm{\lambda}\leq \hat{\bm{y}}+1/\sqrt{\numsamples} \cdot \bm{1}$. We do so because the primal LP often does not yield any solutions {under} the constraint $\hat{\bm{y}} = A\bm{\lambda}$ where $\hat{\bm{y}}$ is empirically calculated using the same number of samples as \AlgDAG. This is mainly because the system of equations can be overspecified, so if we use an estimated version of $\ybf$ that might have some error, two or more systems of equations can result in different values, leading to a contradiction.

{\subsection{Additional Experimental Results}}
\label{sec:large-exp}

As a robustness check on the performance of \AlgDAG, for the synthetic setting, we vary the number of types and the number of items. Specifically, we consider a setting with $15$ frequent types and $20$ items. The total probability $\raresum$ of noise is still chosen from the set $\{0, 0.001, 0.01, 0.05\}$, and $20$ noisy types are generated when $\raresum$ is greater than $0$. The probabilities of frequent and noisy types are generated as those for the first synthetic setting (see Section~\ref{sec:experiment}), i.e., according to a symmetric Dirichlet distribution with coefficient of variation 0.1, multiplied by $(1-\raresum)$ and $\raresum$, respectively. The goal is to estimate the top $n_0 \in \SET{5, 10}$ positions, and we choose $\kappa$ and $\epsilon$ to be $0.01$ and $0.001$, respectively. Results are then averaged over $100$ generated instances.

\begin{table}[htb]
    \centering
    \begin{tabular}{l|c|c|c|c|c|c|c|c}
    \multirow{2}{*}{} & \multicolumn{2}{c|}{Noise $= 0$} & \multicolumn{2}{c|}{Noise $= 0.001$} & \multicolumn{2}{c|}{Noise $= 0.01$} & \multicolumn{2}{c}{Noise $= 0.05$} \\
    \cline{2-9}
    & $\numitems = 5$ & $\numitems = 10$ & $\numitems = 5$ & $\numitems = 10$ & $\numitems = 5$ & $\numitems = 10$ & $\numitems = 5$ & $\numitems = 10$\\
    \hline
    max.~prob. &{$0.0099\pm$} &{$0.0102\pm$} &{$0.0114\pm$} &{$0.0159\pm$} &{$0.0122\pm$} &{$0.0157\pm$} &{$0.0132\pm$} &{$0.0160\pm$}\\
    differences&$0.00045$&$0.00049$&$0.00048$&$0.00050$&$0.00036$&$0.00046$&$0.00042$&$0.00059$\\
    &&&&&&&&\\
    \% of different &{$0.89\%\pm$} &{$1.01\%\pm$} &{$5.10\%\pm$} &{$7.68\%\pm$} &{$6.47\%\pm$} &{$8.97\%\pm$} &{$7.94\%\pm$} &{$9.98\%\pm$}\\
    vertices&$0.272\%$&$0.242\%$&$0.376\%$&$0.388\%$&$0.436\%$&$0.370\%$&$0.580\%$&$0.502\%$\\
    \end{tabular}
    \vspace{0.25cm}
    \caption{The maximum probability discrepancy over all nodes, between the DAG returned by \AlgDAG and the true DAG, and the percentage of vertices that are added wrongly or not inferred by \AlgDAG, averaged over $100$ instances run on synthetic data with $15$ frequent types and $20$ items.}
    \label{tab:resbig}
\end{table}

The experimental results for this setting are summarized in Table~\ref{tab:resbig}. The maximum  probability discrepancy  is $0.0160$, which is obtained when the top $10$ positions are being estimated. The value of this maximum probability discrepancy is comparable with that in the $5$ frequent types setting. Similarly, the percentages of different vertices over different frequent cases are at most $9.98\%$, which is comparable with the result for the cases with $5$ frequent types. The benchmark is not included in this setting because with  a large number of items, the benchmark is {too} computationally inefficient.
\end{APPENDIX}

%\clearpage 
%\input{response_letter.tex}

\iffalse
\newpage
\section{NOT IN THE PAPER: Raw Data and Graphs of the New Experiments}
\input{tex/experiments_raw_data_NOT_PART_OF_PAPER}
\fi

%%%%%%%%%%%%%%%%%
\end{document}